%% file: DCWP.tex
\documentclass[10pt,twocolumn,letterpaper]{article}

\usepackage{cvpr}              

\usepackage{graphicx}
\usepackage{amsmath}
\usepackage{amssymb}
\usepackage{booktabs}
\usepackage[accsupp]{axessibility}  
\usepackage[dvipsnames]{xcolor}

\usepackage{mathtools}
\usepackage{bbm}
\usepackage{verbatim}
\usepackage{array,colortbl,multirow,multicol,booktabs,ctable}
\usepackage{wrapfig}

\usepackage{amsmath, amssymb, mathtools, amsthm, mathtools}
\DeclareMathOperator{\sech}{sech}

\input{macros.tex}

\usepackage{algorithm} 
\usepackage{algpseudocode} 

\DeclareMathOperator{\sgn}{sgn}
\DeclareMathOperator{\sigmoid}{\sigma}
\DeclareMathOperator{\mean}{mean}
\DeclareMathOperator{\norm}{norm}

\usepackage{pifont}
\newcommand{\cmark}{\ding{51}}%
\newcommand{\xmark}{\ding{55}}%

\def\eqref#1{(\ref{#1})}
\usepackage{amsthm}
\newtheorem{theorem}{Theorem}

\newtheorem{lemma}{Lemma}
\newtheorem{proposition}{Proposition}

\newtheorem{innercustomthm}{Theorem}
\newenvironment{customthm}[1]
  {\renewcommand\theinnercustomthm{#1}\innercustomthm}
  {\endinnercustomthm}

\newcommand{\green}[1] {\textcolor{OliveGreen}{#1}} 

\newcommand{\Xb}{{\boldsymbol X}}
\newcommand{\mb}{{\boldsymbol m}}
\newcommand{\Oc}{{\mathcal O}}
\newcommand{\Thetab}{{\boldsymbol \Theta}}
\newcommand{\pibold}{{\boldsymbol \pi}}
\newcommand{\zinv}{Z_{inv}}
\newcommand{\winv}{w_{inv}}
\renewcommand{\xb}{\boldsymbol x}
\renewcommand{\wb}{\boldsymbol w}
\renewcommand{\Zb}{\boldsymbol Z}
\renewcommand{\zb}{\boldsymbol z}

\newcommand{\zsp}{\Zb_{sp}}

\usepackage{url}
\newcommand\given[1][]{\:#1\vert\:}

\usepackage{subcaption}
\usepackage{caption}

\usepackage[pagebackref,breaklinks,colorlinks]{hyperref}

\usepackage[capitalize]{cleveref}
\crefname{section}{Sec.}{Secs.}
\Crefname{section}{Section}{Sections}
\Crefname{table}{Table}{Tables}
\crefname{table}{Tab.}{Tabs.}

\def\cvprPaperID{11236} 
\def\confName{CVPR}
\def\confYear{2023}

\begin{document}

\title{Training Debiased Subnetworks with Contrastive Weight Pruning}

\author{Geon Yeong Park$^{1}$ \qquad Sangmin Lee$^{2}$ \qquad Sang Wan Lee$^{1*}$ \qquad Jong Chul Ye$^{1,2,3*}$ \\
$^{1}$Bio and Brain Engineering, $^{2}$Mathematical Sciences, $^{3}$Kim Jaechul Graduate School of AI \\
Korea Advanced Institute of Science and Technology (KAIST), Daejeon, Korea \\
{\tt\small \{pky3436, leeleesang, sangwan, jong.ye\}@kaist.ac.kr}
}
\maketitle

\begin{abstract}
Neural networks are often biased to spuriously correlated features that provide misleading statistical evidence that does not generalize. This raises an interesting question: ``Does an optimal unbiased functional subnetwork exist in a severely biased network? If so, how to extract such subnetwork?" While empirical evidence has been accumulated about the existence of such unbiased subnetworks, these observations are mainly based on the guidance of ground-truth unbiased samples. Thus, it is unexplored how to discover the optimal subnetworks with biased training datasets in practice. To address this, here we first present our theoretical insight that alerts potential limitations of existing algorithms in exploring unbiased subnetworks in the presence of strong spurious correlations. We then further elucidate the importance of bias-conflicting samples on structure learning. Motivated by these observations, we propose a Debiased Contrastive Weight Pruning (DCWP) algorithm, which probes unbiased subnetworks without expensive group annotations. Experimental results demonstrate that our approach significantly outperforms state-of-the-art debiasing methods despite its considerable reduction in the number of parameters.
\end{abstract}

\section{Introduction}
While deep neural networks have made substantial progress in solving challenging tasks, they often undesirably rely on spuriously correlated features or dataset bias, if present, which is considered one of the major hurdles in deploying models in real-world applications. For example, consider recognizing desert foxes and cats from natural images. If the background scene (e.g., a desert) is spuriously correlated to the type of animal, the neural networks might use the background information as a shortcut to classification, resulting in performance degradation in different backgrounds (e.g., a desert fox in the house). 

{To investigate the origin of the spurious correlations, this paper considers shortcut learning as a fundamental architectural design issue of neural networks. Specifically, if any available information channels in deep networks' structure could transmit the information of spuriously correlated features (\textit{spurious features} from now on), networks would exploit those features as long as they are sufficiently predictive.} It naturally follows that pruning weights on spurious features can purify the biased latent representations, thereby improving performances on bias-conflicting samples\footnote{The \textit{bias-aligned} samples refer to data with a strong correlation between (potentially latent) spurious features and target labels (e.g., cat in the house). The \textit{bias-conflicting} samples refer to the opposite cases where spurious correlations do not exist (e.g., cat in the desert).}. {We conjecture that this neural pruning may improve the generalization of the network in a way that reduces the effective dimension of spurious features, considering that the failure of Out-of-Distribution (OOD) generalization may arise due to high-dimensional spurious features \cite{tsipras2018robustness, nagarajan2020understanding}.} 

Recently, Zhang et al. \cite{zhang2021can} has empirically demonstrated the existence of subnetworks that are less susceptible to spurious features. Based on the modular property of neural networks \cite{csordas2020neural}, they prune out weights that are closely related to the spurious attributes. While \cite{zhang2021can} affords us valuable insights on the importance of neural architectures, the study has limitation in that such neural pruning requires sufficient number of ground-truth bias-conflicting samples. Thus,  \textit{how to discover the optimal subnetworks in practice when the dataset is highly biased?} 

\begin{figure*}[t]
\centering
\includegraphics[width=0.85\textwidth]{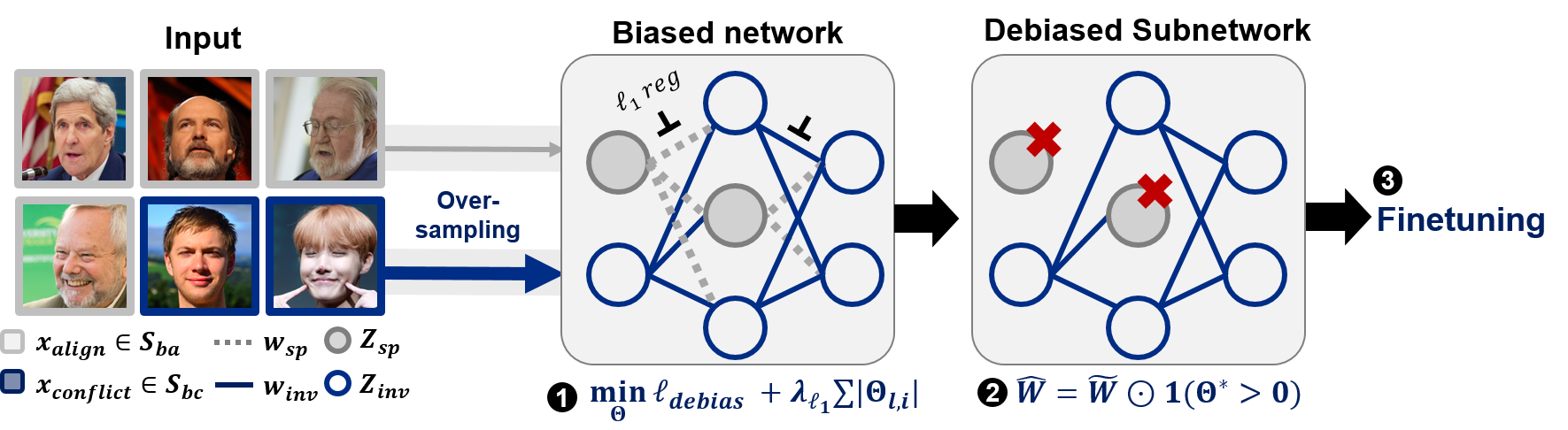} 
\caption{\textbf{Concept}: We demonstrate an inevitable generalization gap of subnetworks obtained by standard pruning methods including \cite{zhang2021can}. Based on these observations, we design a novel subnetwork probing framework by fully exploiting unbiased samples.}
\label{fig:concept}
\end{figure*}

To address this, we first present a simple theoretical observation that reveals the limitations of existing substructure probing methods in searching unbiased subnetworks. Specifically, we reveal that there exists an unavoidable generalization gap in the subnetworks obtained by standard pruning algorithms in the presence of strong spurious correlations. {Our analysis also shows that trained models may inevitably rely on the spuriously correlated features in a practical training setting with finite training time and a number of samples.}

In addition, we show that sampling more bias-conflicting data makes it possible to identify spurious weights. Specifically, bias-conflicting samples require that the weights associated with spurious features should be pruned out as the spurious features do not help predict bias-conflicting samples.
Our theoretical observations suggest that balancing the ratio between the number of bias-aligned and bias-conflicting samples is crucial in finding the optimal unbiased subnetworks.

In practice, the dataset may severely lack diversity for bias-conflicting samples due to the potential pitfalls in data collection protocols or human prejudice. Since it is often highly laborious to supplement enough bias-conflicting samples, 
we propose a novel debiasing scheme called Debiased Contrastive Weight Pruning (DCWP) that uses the oversampled bias-conflicting data to search unbiased subnetworks.

As shown in Fig.~\ref{fig:concept}, DCWP is comprised of two stages: (1) identifying the bias-conflicting samples without expensive annotations on spuriously correlated attributes, and (2) training the pruning parameters to obtain weight pruning masks with  the sparsity constraint and \textit{debiased} loss function.
Here, the debiased loss includes 
a weighted cross-entropy loss for the identified bias-conflicting samples and an alignment loss to further reduce the geometrical alignment gap between bias-aligned and bias-conflicting samples within each class.  

We demonstrate that DCWP consistently outperforms state-of-the-art debiasing methods across various biased datasets, including the Color-MNIST \cite{li2019repair, nam2020learning}, Corrupted CIFAR-10 \cite{hendrycks2019benchmarking}, Biased FFHQ \cite{kim2021biaswap} and CelebA \cite{liu2015deep}, even without direct supervision on the bias type. Our approach improves the accuracy on the unbiased evaluation dataset by \(86.74\% \rightarrow 93.41\%\), \(27.86\% \rightarrow 35.90\%\) on Colored-MNIST and Corrupted CIFAR-10 compared to the second best model, respectively, even when \(99.5\%\) of samples are bias-aligned.


\section{Related works}
\textbf{Spurious correlations.}  A series of empirical works have shown that the deep networks often find shortcut solutions relying on spuriously correlated attributes, such as the texture of image \cite{geirhos2018imagenet}, language biases \cite{gururangan2018annotation}, or sensitive variables such as ethnicity or gender \cite{narayanan2018translation, feldman2015certifying}. Such behavior is of practical concern because it deteriorates the reliability of deep networks in sensitive applications like healthcare, finance, and legal services \cite{corbett2018measure}.

\textbf{Debiasing frameworks.}  Recent studies  to train a debiased network robust to spurious correlations can be roughly categorized into approaches (1) leveraging annotations of spurious attributes, i.e., bias label  \cite{sagawa2019distributionally, wang2020towards},  (2) presuming specific type of bias, e.g., texture \cite{bahng2020learning, ge2021robust} or (3) without using explicit kinds of supervisions on dataset bias \cite{nam2020learning, lee2021learning}. The authors in \cite{sagawa2019distributionally, hu2018does} optimize the worst-group error by using training group information. For practical implementation, reweighting or subsampling protocols are often used with increased model regularization \cite{sagawa2020investigation}. Liu et al.; Sohoni et al. \cite{liu2021just, sohoni2020no} extend these approaches to the settings without expensive group annotations. Goel et al.; Kim et al. \cite{goel2020model, kim2021biaswap} provide bias-tailored augmentations to balance the majority and minority groups. In particular, these approaches have mainly focused on better approximation and regularization of worst-group error combined with advanced data sampling, augmentation, or retraining strategies.

\textbf{Studying impacts of neural architectures.} Recently, the effects of deep neural network architecture on generalization performance have been explored. Diffenderfer et al. \cite{diffenderfer2021winning} employ recently advanced lottery-ticket-style pruning algorithms \cite{frankle2018lottery} to design the compact and robust network architecture. Bai et al. \cite{bai2021ood} directly optimize the neural architecture in terms of accuracy on OOD samples. 
Zhang et al. \cite{zhang2021can} demonstrate the effectiveness of pruning weights on spurious attributes, but the solution for discriminating such spurious weights lacks robust theoretical justifications, resulting in marginal performance gains. To fully resolve the above issues, we carry out a theoretical case study, and build a novel pruning algorithm that distills the representations to be independent of the spurious attributes.

\section{Theoretical insights}
\label{sec: theoretical insights}
\subsection{Problem setup}
Consider a supervised setting of predicting labels \(Y \in \mathcal{Y}\) from input samples \(X \in \mathcal{X}\) by a classifier \(f_{\theta}: \mathcal{X} \rightarrow \mathcal{Y}\) parameterized by \(\theta \in \Theta\). 
Following \cite{zhang2021can},
{let \((X^e, Y^e) \sim P^e\),  where $X^e\in \mathcal{X}$ and $Y^e\in \mathcal{Y}$
refer to  the input random variable and the corresponding label, respectively, and
\(e \in \mathcal{E} = \{1, 2, \dots E\} \) denotes the index of  environment, $P^e$ is the corresponding distribution, and the set \(\mathcal{E}\) corresponds to every possible environments.} 
We further assume that \(\mathcal{E}\) is divided into training environmments \(\mathcal{E}_{train}\) and unseen test environments \(\mathcal{E}_{test}\), i.e. \(\mathcal{E} = \mathcal{E}_{train} \cup \mathcal{E}_{test}\). 

For a given a loss function \(\ell: \mathcal{X} \times \mathcal{Y} \times \Theta \rightarrow \mathbb{R^+}\), the standard training protocol for the empirical risk
minimization (ERM) is to minimize the expected loss with a training environment \(e \in \mathcal{E}_{train}\):
\begin{equation}
\label{3.1 eq: generic_loss}
\hat{\theta}_{ERM} = \arg \min_{\theta} \mathbb{E}_{(X^e, Y^e) \sim \hat{P}^e} \big[\ell(X^e, Y^e; \theta) \big],
\end{equation}
{where \(\hat{P}^e\) is the empirical distribution over the training data.} Our goal is to learn a model with good performance on OOD samples of \(e \in \mathcal{E}_{test}\).

\subsection{Motivating example}
We conjecture that neural networks trained by ERM indiscriminately rely on predictive features, including those spuriously correlated ones \cite{tsipras2018robustness}. 

To verify this conjecture, we present a simple binary-classification example \((\boldsymbol{X}^e, Y^e) \sim P^e\), where \(Y^e \in \mathcal{Y}=\{-1, 1\}\) represents the corresponding target label, and  a sample \(\boldsymbol{X}^e \in \mathcal{X} = \{-1, 1\}^{D+1} \in \mathbb{R}^{D+1}\) is constituted with both the invariant feature \(Z^e_{inv} \in \{-1, 1\}\) and spurious features \(\boldsymbol{Z}^e_{sp} \in \{-1, 1\}^D\), i.e. \(\boldsymbol{X}^e = (Z_{inv}^e, \boldsymbol{Z}_{sp}^e)\). 
Suppose, furthermore, \({Z}_{sp, i}^e\) denote the \(i\)-th spurious feature component of \(\boldsymbol{Z}_{sp}^e\).
Note that we assume \(D \gg 1\) to simulate the model heavily relies on spurious features \(\boldsymbol{Z}_{sp}^e\) \cite{nagarajan2020understanding, zhang2021can}.

We consider the setting where the training environment \(e \in \mathcal{E}_{train}\) is highly biased. In other words, we suppose that \(Z_{inv}^e = Y^e\), and each of the \(i\)-th spurious feature component \({Z}_{sp, i}^e\) is independent and identically distributed (i.i.d) Bernoulli variable: i.e. \({Z}_{sp, i}^e\) independently takes a value equal to \(Y^e\) with a probability \(p^e\) and \(-Y^e\) with a probability \(1-p^e\), where \(p^e \in (0.5, 1], \forall e \in \mathcal{E}_{train}\). Note that \(p^e \rightarrow 1\) as the environment is severely biased. A test environment \(e \in \mathcal{E}_{test}\) is assumed to have \(p^e=0.5\), which implies that the spurious feature is totally independent with \(Y^e\). Then we introduce a linear classifier \(f\) parameterized by a weight vector \(\boldsymbol{w}=(w_{inv}, \boldsymbol{w}_{sp})  \in \mathbb{R}^{D+1}\), where \(w_{inv} \in \mathbb{R}\) and \(\boldsymbol{w}_{sp} \in \mathbb{R}^{D}\). {In this example, we consider a class of pretrained classifiers parameterized by \(\tilde{\boldsymbol{w}}(t) = \big( \tilde{w}_{inv}(t), \tilde{w}_{sp, 1}(t), \dots, \tilde{w}_{sp, D}(t) \big)\), where $t < T$ is a finite pretraining time for some sufficiently large $T$. Time $t$ will be often omitted in notations for simplicity.}

Our goal is to obtain the optimal sparse classifier with a highly biased training dataset. To achieve this, we introduce a binary weight pruning mask \(\boldsymbol{m}\) as \(\boldsymbol{m}=(m_{inv}, \boldsymbol{m}_{sp}) \in \{0, 1\}^{D+1}\) for the pretrained weights, which is a significant departure from the theoretical setting in \cite{zhang2021can}.  Specifically, let \(m_{inv} \sim Bern(\pi_{inv})\),  
where \(\pi_{inv}\) and \(1-\pi_{inv}\) represents the probability of preserving (i.e. $m_{inv}=1$) and pruning out (i.e. $m_{inv}=0$), respectively. Similarly, let ${m}_{sp, i} \sim Bern(\pi_{sp, i}),\forall i$. 
Then, our optimization goal is to estimate the pruning probability parameter 
$\boldsymbol{\pi} = ({\pi}_1, \dots, {\pi}_{D+1})= (\pi_{inv}, \pi_{sp, 1}, \dots, \pi_{sp, D})$, {where \(\boldsymbol{m} \sim P(\boldsymbol{\pi})\) is a mask sampled with probability parameters $\boldsymbol{\pi}$. Accordingly, our main loss function for the pruning parameters given the environment \(e\) can be defined as follows:
\begin{equation}
\label{3.2 eq: loss}
\begin{split}
  \ell^e(\boldsymbol{\pi}) &= \frac{1}{2} \mathbb{E}_{\Xb^e, Y^e, \boldsymbol{m}}[1-Y^e\hat{Y}^e] \\
  &= \frac{1}{2} \mathbb{E}_{\Xb^e, Y^e, \boldsymbol{m}} \left[1-Y^e \cdot \sgn\left( \boldsymbol{\tilde{w}}^T(\boldsymbol{X}^e \odot \boldsymbol{m})\right)\right],
\end{split}
\end{equation}
where \(\hat{Y}^e\) is the prediction of binary classifier, \(\boldsymbol{\tilde{w}}\) is the pretrained weight vector, \(\sgn(\cdot)\) represents the sign function, and \(\odot\) represents element-wise product. 
}


We first derive the upper-bound of the training loss \(\ell^e(\boldsymbol{\pi})\) to illustrate the difficulty of learning optimal pruning parameters in a biased data setting.
The proof can be found in Supplementary Material.
\begin{theorem}
\label{3.2 thm: training upper bound}
(Training and test bound) {Assume that \(p^e > {1} / {2}\) in the biased training environment \(e \in \mathcal{E}_{train}\). Define $\tilde{\boldsymbol{w}}(t)$ as weights pretrained for a finite time $t < T$. Then the upper bound of the error of training environment w.r.t. pruning parameters $\boldsymbol{\pi}$ is given as}: 
\begin{equation}
\label{3.2 eq: training upper bound}
\ell^e(\boldsymbol{\pi}) \leq 2\exp\bigg(-\frac{2\big(\pi_{inv} + (2p^e-1)\sum_{i=1}^{D}  \alpha_i(t) \pi_{sp, i} \big)^2}{4 \sum_{i=1}^D \alpha_i(t)^2 + 1}\bigg),
\end{equation}
{where the weight ratio \(\alpha_i(t) = {\tilde{w}_{sp, i}(t)} / {\tilde{w}_{inv}(t)}\) is bounded below some positive constant. Given a test environment \(e \in \mathcal{E}_{test}\) with \(p^e = \frac{1}{2}\), the upper bound of the error of test environment w.r.t. $\boldsymbol{\pi}$ is given as}:
\begin{equation}
\label{3.2 eq: test upper bound}    
\ell^e(\boldsymbol{\pi}) \leq 2\exp\Big(-\frac{2\pi_{inv}^2}{4 \sum_{i=1}^D \alpha_i(t)^2+1}\Big),
\end{equation}
which implies that there is an unavoidable gap between training bound and test bound.
\end{theorem}

{The detailed proof of Theorem \ref{3.2 thm: training upper bound} is provided in the supplementary material. 
This mismatch of the bounds is attributed to the contribution of \(\pi_{sp, i}\) on the training bound (\ref{3.2 eq: training upper bound}). Intuitively, the networks prefer to \textit{preserve both} \(\tilde{w}_{inv}\) and \(\tilde{w}_{sp, i}\) in the presence of strong spurious correlations due to the inherent sensitivity of ERM to all kinds of predictive features \cite{ilyas2019adversarial, tsipras2018robustness}. This behavior is directly reflected in the training bound, where increasing either \(\pi_{inv}\) or \(\pi_{sp, i}\), i.e., the probability of preserving weights, decreases the training bound. This {inertia} of spurious weights may prevent themselves from being primarily pruned against the sparsity constraint.} 

We note that the unintended reliance on spurious features is fundamentally rooted to the positivity of the weight ratio $\alpha_i(t)$. In the proof of Theorem~\ref{3.2 thm: training upper bound} in Supplementary Material, we show some intriguing properties of $\alpha_i(t)$: (\textbf{1}) If infinitely many data and sufficient training time is provided, the gradient flow converges to the optimal solution which is invariant to $\boldsymbol{Z}^e_{sp}$, i.e., $\alpha_i(t) \rightarrow 0$. In this ideal situation, the gap between training and test bound is closed, thereby guaranteeing generalizations of obtained subnetworks. (\textbf{2}) However, given a finite time $t < T$ with a strongly biased dataset in practice, $\alpha_i(t)$ is bounded below by some positive constant, resulting in an inevitable generalization gap. 

 Theorem \ref{3.2 thm: training upper bound} implies that the classifier may preserve spurious weights due to the lack of bias-conflicting samples, which serve as counterexamples that spurious features themselves fail to explain. It motivates us to analyze the training bound in another environment $\eta$ where we can systematically augment bias-conflicting samples.  
Specifically, consider \(\boldsymbol{X}^\eta = (Z_{inv}^\eta, \boldsymbol{Z}_{sp}^\eta)\),
where \(Z_{inv}^\eta = Y^\eta\) 
and \textit{mixture distribution} of \(\boldsymbol{Z}_{sp}^{\eta}\) given \(Y^\eta=y\) is defined in an element wise as follows:
\begin{equation}
\begin{split}
\label{3.2 eq: mixture distribution}
P_{mix}^{\eta}({Z}_{sp, i}^{\eta} \mid Y^\eta=y) = &\phi P_{debias}^{\eta}({Z}_{sp, i}^{\eta} \mid Y^{\eta}=y) + \\ 
&(1-\phi) P_{bias}^{\eta}({Z}_{sp, i}^{\eta} \mid Y^{\eta}=y),
\end{split}
\end{equation}
where \(\phi\) is a scalar mixture weight,
\begin{equation}
\label{3.2 eq: mixture_debias}
P_{debias}^{\eta}({Z}_{sp, i}^{\eta}\mid Y^{\eta}=y)= 
\begin{cases}
    1, & \text{if } {Z}_{sp, i}^{\eta} = -y\\
    0, & \text{if } {Z}_{sp, i}^\eta = y 
\end{cases}
\end{equation} 
{is a debiasing distribution to weaken the correlation between \(Y^\eta\) and \({Z}_{sp, i}^\eta\) by setting the value of \(Z_{sp,i}^\eta\) as \(-Y^\eta\), and}
\begin{equation}
\label{3.2 eq: mixture_bias}
P_{bias}^\eta({Z}_{sp, i}^\eta\mid Y^\eta=y)= 
\begin{cases}
    p^\eta, & \text{if } {Z}_{sp, i}^\eta = y\\
    1-p^\eta, & \text{if } {Z}_{sp, i}^\eta = -y
\end{cases}
\end{equation}
{is a biased distribution similarly defined in the previous environment \(e \in \mathcal{E}_{train}\).
Given this new environment \(\eta\), the degree of spurious correlations can be controlled by \(\phi\). This leads to a training bound as follow:}

\begin{theorem} 
\label{3.2 thm: mixture training bound}
(Training bound with the mixture distribution) Assume that the defined mixture distribution \(P_{mix}^\eta\) is biased, i.e., for all $i \in \{1, \dots, D\}$,
\begin{equation}
\label{3.2 eq: mixture distribution condition}
  P_{mix}^\eta({Z}_{sp, i}^\eta = -y \mid Y^e=y) \leq P_{mix}^\eta({Z}_{sp, i}^\eta = y \mid Y^\eta=y). 
\end{equation}
Then, \(\phi\) satisfies \(0 \leq \phi \leq 1-\frac{1}{2p^\eta}\). Then the upper bound of the error of training environment \(\eta\) w.r.t. the pruning parameters is given by
\begin{equation}
\begin{split}
\label{3.2 eq: mixture training bound}
&\ell^\eta(\pibold) \leq \\
&2\exp\left(-\frac{2(\pi_{inv} +(2p^\eta(1-\phi)-1)  \sum_{i=1}^D \alpha_i(t) \pi_{sp, i})^2}{4 \sum_{i=1}^D \alpha_{i}(t)^2 + 1}\right).      
\end{split}
\end{equation}
Furthermore, when \(\phi=1-\frac{1}{2p^\eta}\), the mixture distribution is perfectly debiased, and we have
\begin{equation}
\label{3.2 eq: mixture debiased training bound}
\ell^\eta(\pibold) \leq 2\exp\Big(-\frac{2\pi_{inv}^2}{4 \sum_{i=1}^D \alpha_i(t)^2+1}\Big),
\end{equation}
which is equivalent to the test bound in (\ref{3.2 eq: test upper bound}).
\end{theorem}

The detailed proof is provided in the supplementary material. Our new training bound (\ref{3.2 eq: mixture training bound}) suggests that the significance of \(\pi_{sp,i}\) on training bound decreases as \(\phi\) progressively increases, and  at the extreme end  with \(\phi = 1 - \frac{1}{2p^\eta}\), it can be easily shown that \(P_{mix}^\eta({Z}_{sp, i}^\eta \mid Y^\eta = y) = \frac{1}{2}\) for both \(y=1\) and \(y=-1\) so that \({Z}^\eta_{sp, i}\) turns out to be random. 
In other words, by plugging \(\phi = 1 - \frac{1}{2p^\eta}\) into (\ref{3.2 eq: mixture training bound}), we can minimize the gap between training and test error bound, which guarantees the improved OOD generalization.


\section{Debiased Contrastive Weight Pruning}
\label{sec: DCWP}

Our theoretical observations elucidate the importance of balancing between the bias-aligned and bias-conflicting samples in discovering the optimal unbiased subnetworks structure. {While the true analytical form of the debiasing distribution is unknown in practice, we aim to approximate such unknown distribution with existing bias-conflicting samples and simulate the mixture distribution \(P^\eta_{mix}\) with modifying sampling strategy. To this end, we propose a Debiased Contrastive Weight Pruning (DCWP) algorithms that learn the unbiased subnetworks structure from the original full-size network.
} 

Consider a \(L\) layer neural networks as a function \(f_{\boldsymbol{W}}: \mathcal{X} \rightarrow \mathbb{R}^C\) parameterized by weights \(\boldsymbol{W} = \{ \boldsymbol{W}_1, \dots, \boldsymbol{W}_L \} \), where \(C=|\mathcal{Y}|\) is the number of classes. Analogous to the earlier works on pruning, we introduce binary weight pruning masks \(\boldsymbol{m} = \{ \boldsymbol{m}_1, \dots, \boldsymbol{m}_L \} \) to model the subnetworks as \(f(\cdot; \boldsymbol{m}_1 \odot \boldsymbol{W}_1, \dots, \boldsymbol{m}_L \odot \boldsymbol{W}_L)\). We denote such subnetworks as \(f_{\boldsymbol{m} \odot \boldsymbol{W}}\) for the notational simplicity. We treat each entry of \(\boldsymbol{m}_l\) as an independent Bernoulli variable, and model their logits as our new pruning parameters \(\Thetab = \{\Thetab_1, \dots, \Thetab_L\}\) where \(\Thetab_l \in \mathbb{R}^{n_l}\) and \(n_l\) represents the dimensionality of the \(l\)-th layer weights \(\boldsymbol{W}_l\). Then \(\pi_{l, i} = \sigmoid(\Theta_{l, i})\) denotes the probability of preserving the \(i\)-th weight of \(l\)-th layer \(\boldsymbol{W}_{l, i}\) where \(\sigmoid\) refers to a sigmoid function. To enable the end-to-end training, the Gumbel-softmax trick \cite{jang2016categorical} for sampling masks together with \(\ell_1\) regularization term of \(\Thetab\) is adopted as a sparsity constraint. With a slight abuse of notations, \(\boldsymbol{m} \sim G(\Thetab)\) denotes a set of masks sampled with logits \(\Thetab\) by applying Gumbel-softmax trick.

Then our main optimization problem is defined as follows:
\begin{equation}
\label{4 eq: new optimization problem}
\min_{\Thetab} \ell_{debias} \Big( \{(\boldsymbol{x}_i, y_i)\}_{i=1}^{|S|} ; \boldsymbol{\tilde{W}}, \Thetab \Big) + \lambda_{\ell_1} \sum_{l, i}  |\Theta_{l, i}|,
\end{equation}
where \(S\) denotes the index set of whole training samples, \(\lambda_{\ell_1} > 0\) is a Lagrangian multiplier, \(\boldsymbol{\tilde{W}}\) represents the pretrained weights and \(\ell_{debias}\) is our main objective which will be illustrated later. Note that we freeze the pretrained weights \(\boldsymbol{\tilde{W}}\) during training pruning parameters \(\Theta\). We interchangeably use \(\ell_{debias} \Big( \{(\boldsymbol{x}_i, y_i)\}_{i=1}^{|S|} ; \Theta \Big)\) and \(\ell_{debias} \big(S ;\Theta \big)\) in the rest of the paper. For comparison with our formulation, we recast the optimization problem of \cite{zhang2021can} with our notations as follows:

\begin{equation}
\label{4 eq: MRM}
\min_{\Thetab} \ell \Big( \{(\boldsymbol{x}_i, y_i)\}_{i=1}^{|S|} ; \boldsymbol{\tilde{W}}, \Thetab \Big) + \lambda_{\ell_1} \sum_{l, i}  |\Theta_{l, i}|,
\end{equation}
where \cite{zhang2021can} uses the cross entropy (CE) loss function for \(\ell\).

\textbf{Bias-conflicting sample mining} In the first stage, we identify bias-conflicting training samples which empower functional modular probing. Specifically, we train a bias-capturing model and treat an error set \(S_{bc}\) of the index of misclassified training samples as bias-conflicting sample proxies. Our framework is broadly compatible with various bias-capturing models, where we mainly leverage the ERM model trained with generalized cross entropy (GCE) loss \cite{zhang2018generalized}:  

\begin{equation}
\label{4.1 eq: GCE}
\ell_{GCE}(x_i, y_i; \boldsymbol{W}_{B}) =
\frac{1-p_{y_i}(x_i; \boldsymbol{W}_{B})^q}{q},
\end{equation}
{where \(q \in (0, 1]\) is a hyperparameter controlling the degree of bias amplification, \(\boldsymbol{W}_{B}\) is the parameters of the bias-capturing model, and \(p_{y_i}(x_i; \boldsymbol{W}_{B})\) is a softmax output value of the bias-capturing model assigned to the target label \(y_i\).} Compared to the CE loss, the gradient of the GCE loss up-weights the samples with a high probability of predicting the correct target, amplifying the network bias by putting more emphasis on easy-to-predict samples \cite{nam2020learning}.

To preclude the possibility that the generalization performance of DCWP is highly dependent on the behavior of the bias-capturing model, we demonstrate in Section \ref{sec: results} that DCWP is reasonably robust to the degradation of accuracy on capturing bias-conflicting samples. Details about the bias-capturing model and simulation settings are presented in the supplementary material.

\textbf{Upweighting Bias-conflicting samples}
After mining the index set of bias-conflicting sample proxies \(S_{bc}\), we treat \(S_{ba}=S \setminus S_{bc}\) as the index set of majority bias-aligned samples. Then we calculate the weighted cross entropy (WCE) loss \(\ell_{WCE}\big( \{x_i, y_i\}_{i=1}^{|S|}; \boldsymbol{\tilde{W}}, \Theta \big)\) as follows:
\begin{equation}
\begin{split}
\label{4.1 eq: weighted cross entropy}
\ell_{WCE}\Big( S; \boldsymbol{\tilde{W}}, \Theta \Big):=\mathbb{E}_{\boldsymbol{m} \sim G(\Theta)} \big[ &\lambda_{up} \ell_{bc}(S_{bc}; \boldsymbol{m}, \boldsymbol{\tilde{W}}) + \\
&\ell_{ba}(S_{ba}; \boldsymbol{m}, \boldsymbol{\tilde{W}}) \big],
\end{split}
\end{equation}
where \(\lambda_{up} \geq 1\) is an upweighting hyperparameter, and
\begin{equation}
\ell_{bc}(S_{bc}; \boldsymbol{m}, \boldsymbol{\tilde{W}}) = \frac{1}{|S_{bc}|} \sum_{i \in S_{bc}} \ell_{CE}(x_i, y_i; \boldsymbol{m} \odot \boldsymbol{\tilde{W}}),
\end{equation}
where $\ell_{CE}$ denotes the cross entropy loss. $\ell_{ba}$ is defined as similar to $\ell_{bc}$.

The expectation is approximated with Monte Carlo estimates, where the number of mask \(\boldsymbol{m}\) sampled per iteration is set to 1 in practice. To implement (\ref{4.1 eq: weighted cross entropy}), we oversample the samples in \(S_{bc}\) for \(\lambda_{up}\) times more than the samples in \(S_{ba}\). This sampling strategy is aimed at increasing the mixture weight \(\phi\) of the proposed mixture distribution \(P_{mix}^\eta\) in (\ref{3.2 eq: mixture distribution}), while we empirically approximate the unknown bias-conflicting group distribution with the sample set \(S_{bc}\). 

Note that although simple oversampling of bias-conflicting samples may not lead to the OOD generalization due to the inductive bias towards memorizing a few counterexamples in overparameterized neural networks \cite{sagawa2020investigation}, such failure is unlikely reproduced in learning \textit{pruning} parameters under the strong sparsity constraint. {We sample new weight masks $\mb$ for each training iteration in a stochastic manner}, effectively precluding the overparameterized networks from potentially memorizing the minority samples. As a result, DCWP exhibits reasonable performance even with few bias-conflicting samples.

\textbf{Bridging the alignment gap by pruning} To fully utilize the bias-conflicting samples, we consider the sample-wise relation between bias-conflicting samples and majority bias-aligned samples. Zhang et al. \cite{zhang2022correct} demonstrates that the deteriorated OOD generalization is potentially attributed to the distance gap between same-class representations; bias-aligned representations are more closely aligned than bias-conflicting representations, although they are generated from the same-class samples. We hypothesized that well-designed pruning masks could alleviate such geometrical misalignment. Specifically, ideal weight sparsification may guide each latent dimension to be independent of spurious attributes, thereby preventing representations from being misaligned with spuriously correlated latent dimensions. This motivates us to explore pruning masks by contrastive learning. (Related illustrative example in appendix)

Following the conventional notations of contrastive learning, we denote \(f^{enc}_{\boldsymbol{W}}: \mathcal{X} \rightarrow \mathbb{R}^{n_{L-1}}\) as an encoder parameterized by \(\boldsymbol{W} = (\boldsymbol{W}_1, \dots, \boldsymbol{W}_{L-1})\) which maps samples into the representations at penultimate layer. Let \(f^{cls}_{\boldsymbol{W}_{L}}: \mathbb{R}^{n_{L}} \rightarrow \mathbb{R}^C\) be the classification layer parameterized by \(\boldsymbol{W}_{L}\). Then \(f_{\boldsymbol{W}}(\boldsymbol{x}) = f^{cls}_{\boldsymbol{W}_L}(f^{enc}_{\boldsymbol{W}}(\boldsymbol{x})), \forall \boldsymbol{x} \in \mathcal{X}\). We similarly define \(f^{enc}_{\boldsymbol{m} \odot \boldsymbol{W}}\) and \(f^{cls}_{\boldsymbol{m}_L \odot \boldsymbol{W}_L}\). For the \(i\)-th sample \(\boldsymbol{x}_i\), let \(\boldsymbol{z}_i(\boldsymbol{W}) = \norm(f^{enc}_{\boldsymbol{W}}(\boldsymbol{x}_i))\) be the normalized representations lies on the unit hypersphere, and similarly define \(\boldsymbol{z}_i(\boldsymbol{m} \odot \boldsymbol{W})\). We did not consider projection networks \cite{chen2020simple, khosla2020supervised} for architectural simplicity. Given index subsets of training samples \(\mathcal{V}, \mathcal{V}^+\), the supervised contrastive loss \cite{khosla2020supervised} function is defined as follows: 
\begin{equation}
\begin{split}
    \label{4.1 eq: supcon}
    &\ell_{con}(\mathcal{V}, \mathcal{V}^+; \textbf{W}) = \\
    &\sum_{i \in \mathcal{V}} \frac{-1}{|\mathcal{V}^+(y_i)|} \sum_{j \in \mathcal{V}^+(y_i)} \log \frac{\exp\big(\boldsymbol{z}_i(\textbf{W}) \cdot \boldsymbol{z}_j(\textbf{W}) / \tau \big)}{\sum_{a}
    \exp\big(\boldsymbol{z}_i(\textbf{W}) \cdot \boldsymbol{z}_a(\textbf{W}) / \tau\big) },
\end{split}
\end{equation}
where \(a \in \mathcal{V} \setminus \{i\}\), \(\tau > 0\) is a temperature hyperparameter, and \(\mathcal{V}^+(y_i) = \{ k \in \mathcal{V}^+: y_k = y_i, k \neq i \}\) indicates the index set of samples with target label  \(y_i\). Then, we define the debiased alignment loss as follows:
\begin{equation}
\begin{split}
\label{4.1 eq: debiased alignment loss}
&\ell_{align}\Big( \{x_i, y_i\}_{i=1}^{|S|}; \boldsymbol{\tilde{W}}, \Theta \Big) = \mathbb{E}_{\boldsymbol{m} \sim G(\Theta)}\Big[ \\
&\ell_{con}(S_{bc}, S; \boldsymbol{m} \odot \boldsymbol{\tilde{W}}) + 
\ell_{con}(S_{ba}, S_{bc}; \boldsymbol{m} \odot \boldsymbol{\tilde{W}}) \Big],
\end{split}
\end{equation}
where the expectation is approximated with Monte Carlo estimates as in (\ref{4.1 eq: weighted cross entropy}). Intuitively, (\ref{4.1 eq: debiased alignment loss}) reduces the gap between bias-conflicting samples and others (first term), while preventing bias-aligned samples from being aligned too close each other (second term, more discussions in appendix). 

Finally, our debiased loss in (\ref{4 eq: new optimization problem}) is defined as follows:
\begin{equation}
\begin{split}
\label{4.1 eq: debias loss}
\ell_{debias}\Big( S; \boldsymbol{\tilde{W}}, \Theta \Big) = &\ell_{WCE}\Big( S; \boldsymbol{\tilde{W}}, \Theta \Big) + \\
&\lambda_{align} \ell_{align}\Big( S; \boldsymbol{\tilde{W}}, \Theta \Big),
\end{split}
\end{equation}
where \(\lambda_{align}>0\) is a balancing hyperparameter. 

\textbf{Fine-tuning after pruning} After solving (\ref{4 eq: new optimization problem}) by gradient-descent optimization, we can obtain the pruning parameters  \(\Thetab^{*}\). This allows us to uncover the structure of unbiased subnetworks with binary weight masks \(\boldsymbol{m}^* = \{\boldsymbol{m}_1^*, \dots, \boldsymbol{m}_L^*\}\), where \(\boldsymbol{m}^*_l = \{ \mathbbm{1}(\sigmoid(\Theta^*_{l, i}) > 1/2) \given[] 1 \leq i \leq n_l \}, \forall l \in \{1, \dots, L\}\), and \(n_l\) is a dimensionality of the \(l\)-th weight. After pruning, we finetune the survived weights \(\boldsymbol{\hat{W}} =  \boldsymbol{m}^* \odot \boldsymbol{\tilde{W}}\)  using \(\ell_{WCE}\) in (\ref{4.1 eq: weighted cross entropy}) and \(\lambda_{align} \ell_{align}\) in (\ref{4.1 eq: debiased alignment loss}).
Interestingly, we empirically found that the proposed approach works well without the reset \cite{frankle2018lottery} (Related experiments in Section \ref{sec: results}). Accordingly, we resume the training while fixing the unpruned pretrained weights. The pseudo-code of DCWP is provided in Algorithm 1.

\begin{algorithm}
\label{pseudocode}
	\caption{Debiased Contrastive Weight Pruning (DCWP)} 
	\begin{algorithmic}[1]
	    \State {\bfseries Input:} Dataset \(D=\{(x_i, y_i)_{i=1}^{|S|}\}\), pruning parameters \(\Theta\), Training iterations \(T_1, T_2, T_3\).
        \State {\bfseries Output:} Trained pruning parameters \(\Theta^*\) and finetuned weights \(\boldsymbol{W}^*\) \\
        
        \State \textbf{Stage 1.} \textit{Mining debiased samples}
            \State Update the weights of bias-capturing network \(\boldsymbol{W}_b\) on \(D\) for \(T_1\) iterations.
            \State Identify \(S_{bc}\) and \(S_{ba}\). \\
        
        \State \textbf{Stage 2.} \textit{Debiased Contrastive Weight Pruning}
        
        \State Pretrain the main network on \(D\). Denote the pretrained weights as \(\boldsymbol{\tilde{W}}\).
        
		\For {$t=1$ \textbf{to} $T_2$}
		    \State Update \(\Theta\) with \(\ell_{debias} \Big(S; \boldsymbol{\tilde{W}}, \Theta \Big) + \lambda_{\ell_1} \sum_{l, i}  |\Theta_{l, i}|\) as in (\ref{4 eq: new optimization problem}). 
		\EndFor
		
        \State Prune out weight as \(\boldsymbol{\hat{W}} = \boldsymbol{\tilde{W}} \odot \mathbbm{1}(\Theta^* > 0)\).
		\State Update \(\boldsymbol{\hat{W}}\) with \(\ell_{WCE}\) and \(\lambda_{align} \ell_{align}\) on D for \(T_3\) iterations.
		
	\end{algorithmic} 
\end{algorithm}


\begin{table}[htbp]
\caption{Unbiased test accuracy evaluated on CMNIST, CIFAR10-C and bias-conflict test accuracy evaluated on BFFHQ. Models requiring supervisions on dataset bias are denoted with \cmark, while others are denoted with \xmark. Results are averaged on 4 different random seeds.} 
\label{5. table: full}
\centering
\resizebox{0.5\textwidth}{!}{
\begin{tabular}{c c c c c c c c c}
\toprule
\multirow{2}{*}{Dataset} & \multirow{2}{*}{Ratio (\%)} & {ERM} & {EnD} & {Rebias} & {MRM} & {LfF} & {DisEnt} & {DCWP} \\
\cmidrule{3-9}
{} & {} & {\xmark} & {\cmark} & {\cmark} & {\xmark} & {\xmark} & {\xmark} & {\xmark} \\ 
\midrule
\multirow{4}{*}{CMNIST} & {0.5} & {62.36} & {84.32} & {69.12} & {60.98} & {83.73} & {86.74} & {\textbf{93.41}}   \\
                    {} & {1.0} & {81.73} & {94.98} & {84.65} & {80.42} & {88.44} & {93.15} & {\textbf{95.98}} \\
                    {} & {2.0} & {89.33} & {97.01} & {91.96} & {89.31} & {92.67} & {95.15} & {\textbf{97.16}} \\
                    {} & {5.0} & {95.22} & {98.00} & {96.74} & {95.23} & {94.90} & {96.76} & {\textbf{98.02}} \\
\midrule
\multirow{4}{*}{CIFAR10-C} & {0.5} & {22.02} & {23.93} & {21.73} & {23.92} & {27.02} & {27.86} & {\textbf{35.90}}   \\
                    {} & {1.0} & {28.00} & {27.61} & {28.09} & {27.77} & {31.44} & {34.62} & {\textbf{41.56}} \\
                    {} & {2.0} & {34.63} & {36.62} & {35.57} & {33.53} & {38.49} & {41.95} & {\textbf{49.01}} \\
                    {} & {5.0} & {45.66} & {43.67} & {48.22} & {47.00} & {46.16} & {49.15} & {\textbf{56.17}} \\
\midrule
\multirow{1}{*}{BFFHQ} & {0.5} & {52.25} & {59.80} & {54.90} & {54.75} & {56.50} & {55.50} & {\textbf{60.35}}   \\
\bottomrule
\end{tabular}
}
\end{table}

\begin{table}[htbp]
\caption{Worst-group and average test accuracies on CelebA (Blonde). (\cmark, \xmark) here represents $\texttt{Idx}=(6, 4)$ (w/ and w/o pruning) in Table \ref{5. table: ablation}, respectively, which shows the impacts of pruning.} 
\label{re: celeba}
\centering
\resizebox{0.42\textwidth}{!}{
\begin{tabular}{c c c c c c}
\toprule
{Models} & {ERM} & {DisEnt} & {JTT \cite{liu2021just}} & {DCWP (\xmark)} & {DCWP (\cmark)} \\
\midrule
{Worst-group} & {47.02} & {65.26} & {76.80} & {67.85} & {\textbf{79.30}} \\
\midrule
{Average} & {\textbf{97.80}} & {67.88} & {93.98} & {95.89} & {94.50} \\
\bottomrule
\end{tabular}
}
\end{table}

\section{Experimental results}
\label{sec: results}

\subsection{Methods}
\noindent\textbf{Datasets} To show the effectiveness of the proposed pruning algorithms, we evaluate the generalization performance of several debiasing approaches on Colored MNIST (CMNIST), Corrupted CIFAR-10 (CIFAR10-C), Biased FFHQ (BFFHQ) with varying ratio of bias-conflicting samples, i.e., bias ratio. We report unbiased accuracy \cite{nam2020learning, lee2021learning} on the test set, which includes a balanced number of samples from each data group. We also report bias-conflict accuracy for some experiments, which is the average accuracy on bias-conflicting samples included in an unbiased test set. Specifically, we report the bias-conflict accuracy on BFFHQ in which half of the unbiased test samples are bias-aligned, while the model with the best-unbiased accuracy is selected (Unbiased accuracy in Table \ref{5. table: robust}). For CelebA (blonde) \cite{sagawa2019distributionally, hong2021unbiased}, we report worst-group and average accuracy following \cite{sagawa2019distributionally} considering that abundant samples are included in (\texttt{Blonde Hair=0}, \texttt{Male=0}) bias-conflicting group. We use the same data splits from \cite{hong2021unbiased}.


\noindent\textbf{Baselines} We compare DCWP with vanilla network trained by ERM, and the following state-of-the-art debiasing approaches: EnD \cite{tartaglione2021end}, Rebias \cite{bahng2020learning}, MRM \cite{zhang2021can}, LfF \cite{nam2020learning}, JTT \cite{liu2021just} and DisEnt \cite{lee2021learning}. EnD relies on the annotations on the spurious attribute of training samples, i.e., bias labels. Rebias relies on prior knowledge about the type of dataset bias (e.g., texture). MRM, LfF, JTT and DisEnt do not presume such bias labels or prior knowledge about dataset bias. Notably, MRM is closely related to DCWP where it probes the unbiased functional subnetwork with standard cross entropy. Details about other simulation settings are provided in Supplementary Material.

\subsection{Evaluation results}
As shown in Table \ref{5. table: full}, we found that DCWP outperforms other state-of-the-art debiasing methods by a large margin. Moreover, the catastrophic pitfalls of the existing pruning method become evident, where MRM fails to search for unbiased subnetworks. It underlines that the proposed approach for utilizing bias-conflicting samples plays a pivotal role in discovering unbiased subnetworks.

\subsection{Quantitative analyses}
\noindent\textbf{Ablation studies} To quantify the extent of performance improvement achieved by each introduced module, we analyzed the dependency of model performance on: (a) pruning out spurious weights following the trained parameters, (b) using alignment loss or (c) oversampling identified bias-conflicting samples when training \(\Thetab\) and \(\boldsymbol{\hat{W}}\). To emphasize the contribution of each module, we intentionally use an SGD optimizer which results in lower baseline accuracy (and for other CMNIST experiments in this subsection as well). Table \ref{5. table: ablation} shows that every module plays an important role in OOD generalization, while (a) pruning contributes significantly comparing (1$\rightarrow$2, \green{+7.19}$\%$), (3$\rightarrow$5, \green{+11.59}$\%$) or (4$\rightarrow$6, \green{+8.68}$\%$).

\begin{table}[htbp]
\caption{Ablation study on CMNIST (Bias ratio=1\(\%\)). Unbiased accuracy is reported. $\texttt{Idx}=2$ uses $\ell_{WCE}$ only for training pruning parameters $\Theta$ while using $\ell_{CE}$ for retraining. $\texttt{Idx}=3, 4$ does not conduct pruning and finetune the pretrained weights $\boldsymbol{\tilde{W}}$ by oversampling minorities or using alignment loss. } 
\label{5. table: ablation}
\centering
\resizebox{0.4\textwidth}{!}{
\begin{tabular}{c c c c c c}
\toprule
{\texttt{Idx}} & {(a) Pruning} & {(b) \(\ell_{align}\)} & {(c) \(\ell_{WCE}\)} & {Accuracy (\%)} \\
\midrule
{1} & {-} & {-} & {-} & {43.10} \\
{2} & {\cmark} & {-} & {-} & {50.29} \\
{3} & {-} & {-} & {\cmark} & {73.20} \\
{4} & {-} & {\cmark} & {\cmark} & {79.28} \\
{5} & {\cmark} & {-} & {\cmark} & {84.79} \\
{6} & {\cmark} & {\cmark} & {\cmark} & {\textbf{87.96}} \\
\bottomrule
\end{tabular}
}
\end{table}

\noindent\textbf{Dependency on bias-capturing models} To evaluate the reliability of DCWP, we compare different version of DCWP which does not rely on the dataset-tailored mining algorithms. We posit that early stopping \cite{liu2021just} is an easy plug-and-play method to train the bias-capturing model in general. Thus we newly train DCWP\(_{ERM}\) which collects bias-conflicting samples by using the early-stopped ERM model. Table \ref{5. table: robust} shows that DCWP\(_{ERM}\) outperforms other baselines even though the precision, the fraction of samples in \(S_{bc}\) that are indeed bias-conflicting, or recall, the fraction of the bias-conflicting samples that are included in \(S_{bc}\), were significantly dropped. It implies that DCWP may perform reasonably well with the limited number and quality of bias-conflicting samples.


\begin{table}[htbp]
\caption{Robustness dependency of DCWP on the performance of bias-capturing models. We set the bias ratio as 1$\%$ for CIFAR10-C. Results are averaged on 4 different random seeds.} 
\label{5. table: robust}
\centering
\resizebox{0.5\textwidth}{!}{
\begin{tabular}{c c c c c c c}
\toprule
\multirow{2}{*}{Dataset} & \multirow{2}{*}{Model} & \multicolumn{3}{c}{Accuracy} & \multicolumn{2}{c}{Mining metrics} \\
\cmidrule{3-7}
{} & {} & {bias-align} & {bias-conflict} & {unbiased} & {precision} & {recall} \\ 
\midrule
\multirow{3}{*}{CIFAR10-C} & {DisEnt} & {80.04} & {26.51} & {34.62} & {-} & {-} \\
                    {} & {DCWP\(_{ERM}\)} & {\textbf{94.33}} & {\underline{29.75}} & {\underline{36.21}} & {19.71} & {79.53} \\
                    {} & {DCWP} & {\underline{91.68}} & {\textbf{35.99}} & {\textbf{41.56}} & {85.97} & {74.89} \\
\midrule
\multirow{4}{*}{BFFHQ} & {DisEnt} & {89.80} & {55.55} & {72.68} & {-} & {-} \\
                    {} & {LfF} & {96.05} & {56.50} & {76.30} & {-} & {-} \\
                    {} & {DCWP\(_{ERM}\)} & {\textbf{99.45}} & {\underline{56.90}} & {\underline{78.20}} & {20.18} & {28.39} \\
                    {} & {DCWP} & {\underline{98.85}} & {\textbf{60.35}} & {\textbf{79.60}} & {30.61} & {31.25} \\
\bottomrule
\end{tabular}
}
\end{table}

\begin{figure}[!hbt]
\centering
\begin{subfigure}[c]{0.3\textwidth}
\includegraphics[width=\textwidth]{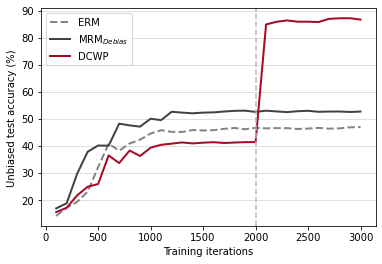} 
\caption{Weight reset}\label{fig: reinit}
\end{subfigure}
\begin{subfigure}[c]{0.35\textwidth}
\includegraphics[width=\textwidth]{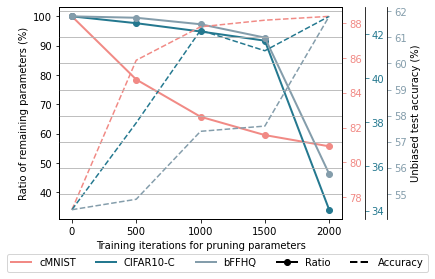}
\caption{Training iterations for \(\Theta\)}\label{fig: pruning ratio}
\end{subfigure}
\caption{(\textbf{a}) Comparison study on finetuning and weight resetting (CMNIST, bias ratio=\(1\%\)). For DCWP, after pretraining weights for 2000 iterations, we pause and start training pruning parameters (vertical dotted line in the figure). After convergence, we mask out and finetune weights for another 1000 iterations. For MRM\(_{debias}\), we reset the unpruned weight to its initialization and retrain for 3000 iterations. (\textbf{b}) Sensitivity analysis on the training iterations for pruning parameter \(\Theta\). Bias ratio=1\(\%\) for both CMNIST and CIFAR10-C. Bias-conflict accuracy is reported for BFFHQ.} \label{fig: analyses}
\end{figure}

\textbf{Do we need to reset weights?} While it becomes widespread wisdom that remaining weights should be reset to their initial ones from the original network after pruning \cite{frankle2018lottery}, we analyze whether such reset is also required for the proposed pruning framework. We compared the training dynamics of different models such as: (1) ERM model, (2) MRM\(_{debias}\) which solves (\ref{4 eq: new optimization problem}) instead of (\ref{4 eq: MRM}) to obtain the weight pruning masks, and (3) DCWP. Note that MRM\(_{debias}\) reset the unpruned weights to its initialization after pruning. Figure \ref{fig: reinit} shows that although MRM\(_{debias}\) makes a considerable advance, weight reset inevitably limits the performance gain. Moreover, finetuning the biased model significantly improves the generalization performance within only a few iterations, which implies that the proposed neural pruning can further boost the accuracy without weight reset. This finding allows us to debias large-scale pretrained models \textit{without} retraining by simple pruning and finetuning.

\begin{figure}[t]
\centering
\includegraphics[width=0.47\textwidth]{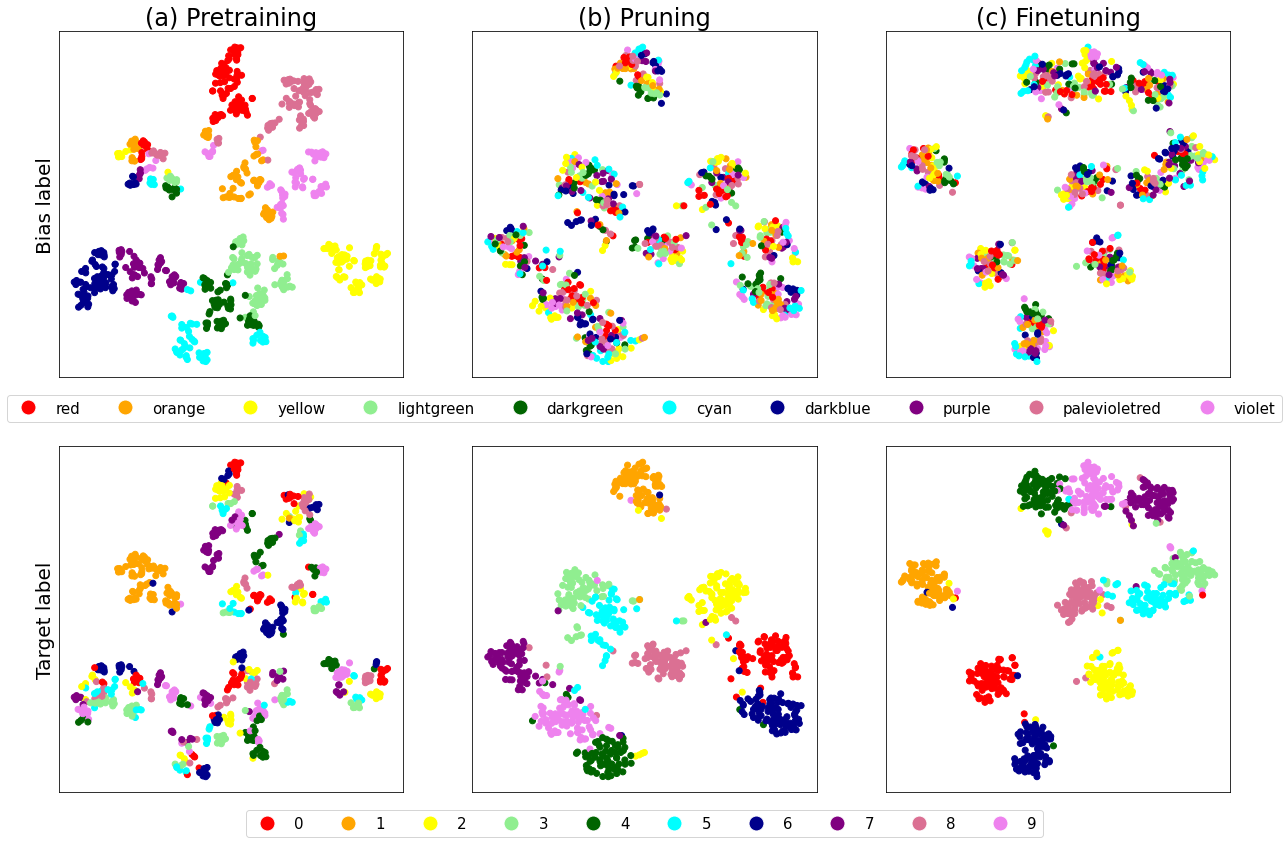} 
\caption{t-SNE visualization of representations encoded from unbiased test samples after (\textbf{a}) pretraining, (\textbf{b}) pruning and (\textbf{c}) finetuning (CMNIST, bias ratio=0.5$\%$). Each point is painted following its label (i.e., bias label in first row, and target label in second row).}
\label{fig: tsne}
\end{figure}

\textbf{Sensitivity analysis on training iterations} We also analyzed the hyperparameter sensitivity on the training iterations of the pruning parameter \(\Thetab\). The unbiased test accuracy is evaluated with weight pruning masks generated by \(\Thetab\) trained for \(\{500, 1000, 1500, 2000\}\) iterations on each dataset. Figure \ref{fig: pruning ratio} shows that the accuracy increases as more (potentially biased) weights are pruned out. It implies that the proposed method can compress the networks to a substantial extent while significantly improving the OOD generalization performance.

\textbf{Visualization of learned latent representations.} We visualized latent representations of unbiased
test samples in CMNIST after (a) pretraining, (b) pruning, and (c) finetuning. Note that we did
not reset or finetune the weights in (b). As reported in Figure \ref{fig: tsne}, biased representations in (a) are misaligned along with bias labels as discussed in section 4. However, after pruning, the representations were well-aligned with respect to the class of digits even without modifying the values of pretrained weights. It implies that the geometrical misalignment of representations can be addressed
by pruning spurious weights while finetuning with $\ell_{debias}$ can further improve the generalizations.

\section{Conclusion} This paper presented a novel functional subnetwork probing method for OOD generalization. Our goal was to find a winning functional lottery ticket \cite{zhang2021can}, which can achieve better OOD performance compared to its counterpart full network, given a highly biased dataset in practice. We provided theoretical insights and empirical evidence to show that the minority samples provide an important clue for probing the optimal unbiased subnetworks. Simulations on various benchmark datasets demonstrated that our model significantly outperforms state-of-the-art debiasing methods. The proposed method is memory efficient and potentially compatible with many other debiasing methods.

\subsubsection*{Acknowledgments}
This work was supported by the National Research Foundation of Korea (NRF) grant funded by the Korea government (MSIT) (NRF-2020R1A2B5B03001980), Institute for Information \& Communications Technology Planning \& Evaluation (IITP) grant funded by the Korea government (No. 2017-0-00451, No. RS-2023-00233251, System3 reinforcement learning with high-level brain functions), KAIST Key Research Institute (Interdisciplinary Research Group) Project and Field-oriented Technology Development Project for Customs Administration through National Research Foundation of Korea (NRF) funded by the Ministry of Science \& ICT and Korea Customs Service(**NRF-2021M3I1A1097938**).


\clearpage
\onecolumn
\noindent\makebox[\textwidth][c]{\Large\bfseries Appendix: Training Debiased Subnetworks with Contrastive Weight Pruning}

The supplementary material is organized as follows. We first present the proof for Theorem \ref{supple thm: training upper bound} and \ref{supple thm: mixture training bound}. In section \ref{sec: misalignment}, we extend the presented theoretical example in the main paper to illustrate the risks of geometrical misalignment of embeddings arising from strong spurious correlations. Additional results are reported in section \ref{sec: additional results}. Optimization setting, hyperparameter configuration, and other experimental details are provided in section \ref{sec: experimental setup}.

\section{Proofs}
\label{sec: proofs}

In this section, we present the detailed proofs for Theorems \ref{supple thm: training upper bound} and \ref{supple thm: mixture training bound} explained in the main paper, followed by an illustration about the dynamics of weight ratio $\alpha_i(t) = \tilde{w}_{sp, i}(t) / \tilde{w}_{inv}(t)$.

\subsection{Proof of Theorem 1}

\begin{customthm}{1}
\label{supple thm: training upper bound}
(Training and test bound) {Assume that \(p^e > {1} / {2}\) in the biased training environment \(e \in \mathcal{E}_{train}\). Define $\tilde{\boldsymbol{w}}(t)$ as weights pretrained for a finite time $t < T$. Then the upper bound of the error of training environment w.r.t. pruning parameters $\boldsymbol{\pi}$ is given as}: 

\begin{equation}
\label{supple eq: training upper bound}
\ell^e(\boldsymbol{\pi}) \leq 2\exp\bigg(-\frac{2\big(\pi_{inv} + (2p^e-1)\sum_{i=1}^{D}  \alpha_i(t) \pi_{sp, i} \big)^2}{4 \sum_{i=1}^D \alpha_i(t)^2 + 1}\bigg),
\end{equation}
{where the weight ratio \(\alpha_i(t) = {\tilde{w}_{sp, i}(t)} / {\tilde{w}_{inv}(t)}\) is bounded below some positive constant. Given a test environment \(e \in \mathcal{E}_{test}\) with \(p^e = \frac{1}{2}\), the upper bound of the error of test environment w.r.t. $\boldsymbol{\pi}$ is given as}:
\begin{equation}
\label{supple eq: test upper bound}    
\ell^e(\boldsymbol{\pi}) \leq 2\exp\Big(-\frac{2\pi_{inv}^2}{4 \sum_{i=1}^D \alpha_i(t)^2+1}\Big),
\end{equation}
which implies that there is an unavoidable gap between training bound and test bound.
\end{customthm}

\begin{proof}
We omit time $t$ in $\tilde{\boldsymbol{w}}(t)$ and $\alpha_i(t)$ for notational simplicity throughout the proof of Theorem \ref{supple thm: training upper bound} and \ref{supple thm: mixture training bound}.

We recall the loss function defined in the main paper for convenience. 
\begin{equation}
\label{supple eq: loss}
\begin{split}
  \ell^e(\boldsymbol{\pi}) &= \frac{1}{2} \mathbb{E}_{\Xb^e, Y^e, \boldsymbol{m}}[1-Y^e\hat{Y}^e] \\
  &= \frac{1}{2} \mathbb{E}_{\Xb^e, Y^e, \boldsymbol{m}} \left[1-Y^e \cdot \sgn\left( \boldsymbol{\tilde{w}}^T(\boldsymbol{X}^e \odot \boldsymbol{m})\right)\right],
\end{split}
\end{equation}
where \(\hat{Y}^e\) is the prediction of binary classifier, \(\boldsymbol{\tilde{w}}\) is the pretrained weight vector, \(\sgn(\cdot)\) represents the sign function, and \(\odot\) represents element-wise product.

The prediction from the classifier \(\hat{Y}^e\) is defined as 
\begin{equation}
\label{supple eq: yhat}
\begin{split}
    \hat{Y}^e &= \sgn\big( \boldsymbol{\tilde{w}}^T(\boldsymbol{X}^e \odot \boldsymbol{m})\big) \\
    &=\sgn\Big(\Oc^e\Big),
\end{split}
\end{equation}
where 
\begin{equation}
\label{supple eq: Oc}
\Oc^e:= \tilde{w}_{inv} m_{inv} Z_{inv}^e + \sum_{i=1}^D \tilde{w}_{sp, i} {m}_{sp, i} {Z}_{sp, i}^e.
\end{equation}

Assume that \(Y^e\) is uniformly distributed binary random variable. Then,
\begin{equation}
\label{supple eq: conditional loss}
\begin{split}
    \mathbb{E}_{\boldsymbol{X}^e, Y^e, \mb} [Y^e\hat{Y}^e] &=  \frac{1}{2}   \mathbb{E}_{\boldsymbol{X}^e, \mb} \Big[\hat{Y}^e | Y^e = 1 \Big]-\frac{1}{2}  \mathbb{E}_{\boldsymbol{X}^e, \mb} \Big[\hat{Y}^e | Y^e = -1 \Big],
\end{split}
\end{equation}
where
\begin{equation}
\label{supple eq: y=1 loss}
\begin{split}
        \mathbb{E}_{\boldsymbol{X}^e, \mb} \Big[\hat{Y}^e | Y^e = 1 \Big] &= \mathbb{E}_{\boldsymbol{X}^e, \mb} \Big[\sgn\Big( \Oc^e \Big) \given[\Big] Y^e = 1 \Big] \\
        &= P\big( \Oc^e >0 \given[\big] Y^e=1  \big) - P\big(\Oc^e<0 \given[\big] Y^e=1 \big) \\ 
        &= 1-2P \big( \Oc^e <0 \given[\big] Y^e=1 \big),
\end{split}
\end{equation}
and

\begin{equation}
\label{supple eq: y=-1 loss}
\begin{split}
        \mathbb{E}_{\boldsymbol{X}^e, \mb} \Big[\hat{Y}^e | Y^e = -1 \Big]
        &= P\big( \Oc^e >0 \given[\big] Y^e=-1  \big) - P\big(\Oc^e<0 \given[\big] Y^e=-1 \big) \\ 
        &= -\mathbb{E}_{\boldsymbol{X}^e, \mb} \Big[\hat{Y}^e | Y^e = 1 \Big], 
\end{split}
\end{equation}
where we use \(P \big(  \Oc^e <0 \given[\big] Y^e=1 \big) = P \big( \Oc^e > 0 \given[\big] Y^e=-1 \big)\)
and \(P \big(  \Oc^e >0 \given[\big] Y^e=1 \big) = P \big( \Oc^e < 0 \given[\big] Y^e=-1 \big)\)
thanks to the symmetry.
Therefore, we have
\begin{equation}
\label{supple eq: training error}
\begin{split}
\ell^e(\boldsymbol{\pi}) 
    &=\frac{1}{2}\mathbb{E}_{\boldsymbol{X}^e, Y^e, \mb} [1 - Y^e\hat{Y}^e] \\
    &=\frac{1}{2} -\frac{1}{2}\mathbb{E}_{\boldsymbol{X}^e, \mb} \Big[\hat{Y}^e | Y^e = 1 \Big]\\
    &=P \big( \Oc^e <0 \given[\big] Y^e=1 \big).
\end{split}
\end{equation}

{In order to derive a concentration inequality of $\ell^e(\pibold)$, we compute a conditional expectation as follows:
\begin{equation}
\label{supple eq: expectation of Z bar}
\begin{split}
    \mathbb{E}_{\boldsymbol{X}^e, \boldsymbol{m}} \big[ \Oc^e \given[\big] Y^e=1 \big] 
    &= \mathbb{E}_{\boldsymbol{X}^e, \boldsymbol{m}} \Big[ \tilde{w}_{inv} m_{inv} Z_{inv}^e + \sum_{i=1}^D \tilde{w}_{sp, i} m_{sp, i} Z_{sp, i}^e  \given[\Big] Y^e=1 \Big] \\
    &= \mathbb{E}_{\boldsymbol{X}^e, \boldsymbol{m}} \Big[ \tilde{w}_{inv} m_{inv} + 
             \sum_{i=1}^D \tilde{w}_{sp, i} m_{sp, i} Z_{sp, i}^e \given[\Big] Y^e=1 \Big] \\
    &= \tilde{w}_{inv}\pi_{inv} +  \mathbb{E}_{\boldsymbol{X}^e, \boldsymbol{m}} \bigg[\sum_{i=1}^D \tilde{w}_{sp, i} m_{sp, i} Z_{sp, i}^e \given[\Big] Y^e=1 \bigg] \\
    &= \tilde{w}_{inv} \pi_{inv} + \sum_{i=1}^D (2p^e-1) \tilde{w}_{sp, i} \pi_{sp, i},
\end{split}
\end{equation}
where the last equality follows from the independence of \(Z_{sp, \cdot}\) and \(m_{sp, \cdot}\) as assumed in the main paper. Then, 
\begin{equation}
\label{supple eq: naive hoeffding}
\begin{split}
    P\big( \Oc^e <0 \given[\big] Y^e=1 \big) 
    &= P\Big( \Oc^e - \mathbb{E}_{\boldsymbol{X}^e, \textbf{m}} \big[ \Oc^e \big] 
    < - \mathbb{E}_{\boldsymbol{X}^e, \boldsymbol{m}} \big[ \Oc^e \big]  \given[\big] Y^e=1 \Big) \\
    &\leq P\Big( \given[\Big] \Oc^e - \mathbb{E}_{\boldsymbol{X}^e, \boldsymbol{m}} \big[ \Oc^e \big] \given[\Big] > 
    \mathbb{E}_{\boldsymbol{X}^e, \boldsymbol{m}} \big[ \Oc^e \big] \given[\big] Y^e=1 \Big) \\
    &\leq 2\exp\bigg(-\frac{2 \mathbb{E}_{\boldsymbol{X}^e, \boldsymbol{m}} \big[ \Oc^e \given[\big] Y^e=1 \big]^2 }{ \tilde{w}_{inv}^2 + \sum_{i=1}^D 4\tilde{w}_{sp, i}^2}\bigg) \\
    &\leq 2\exp\bigg(-\frac{2\big( \tilde{w}_{inv} \pi_{inv} + \sum_{i=1}^D  (2p^e-1) \tilde{w}_{sp, i} \pi_{sp, i} \big)^2}{\tilde{w}_{inv}^2 + \sum_{i=1}^D 4\tilde{w}_{sp, i}^2}\bigg) \\
    &\leq 2\exp\bigg(-\frac{2\big( \pi_{inv} + \sum_{i=1}^D (2p^e-1) \alpha_{i} \pi_{sp, i} \big)^2}
    {1 + \sum_{i=1}^D 4\alpha_i^2}\bigg),
\end{split}
\end{equation}
where the second inequality is obtained using Hoeffding's inequality, third inequality is from (\ref{supple eq: expectation of Z bar}), and last inequality is obtained by dividing both denominator and numerator with $\tilde{w}_{inv}^2$. We use the definition of weight ratio $\alpha_i = \tilde{w}_{sp, i} / \tilde{w}_{inv}$. For the second inequality, we use that $\tilde{w}_{inv} m_{inv} Z^e_{inv} \in \{0, \tilde{w}_{inv}\}$ and $\tilde{w}_{sp, i} m_{sp, i} Z^e_{sp, i} \in \{-\tilde{w}_{sp, i}, 0, \tilde{w}_{sp, i}\} ~\forall i$ in (\ref{supple eq: Oc}) to obtain the denominator.}

Finally, the proof for the positivity of $\alpha_i(t)$ comes from Proposition 1 in section 1.3 in this appendix.
This concludes the proof.
\end{proof}

\subsection{Proof of Theorem 2}

\begin{customthm}{2}
\label{supple thm: mixture training bound}
(Training bound with the mixture distribution) Assume that the defined mixture distribution \(P_{mix}^\eta\) is biased, i.e., for all $i \in \{1, \dots, D\}$,
\begin{equation}
\label{3.2 eq: mixture distribution condition}
  P_{mix}^\eta({Z}_{sp, i}^\eta = -y \mid Y^e=y) \leq P_{mix}^\eta({Z}_{sp, i}^\eta = y \mid Y^\eta=y). 
\end{equation}
Then, \(\phi\) satisfies \(0 \leq \phi \leq 1-\frac{1}{2p^\eta}\). Then the upper bound of the error of training environment \(\eta\) w.r.t. the pruning parameters is given by
\begin{equation}
\begin{split}
\label{3.2 eq: mixture training bound}
\ell^\eta(\pibold) \leq 2\exp\left(-\frac{2(\pi_{inv} +(2p^\eta(1-\phi)-1)  \sum_{i=1}^D \alpha_i(t) \pi_{sp, i})^2}{4 \sum_{i=1}^D \alpha_{i}(t)^2 + 1}\right).   
\end{split}
\end{equation}
Furthermore, when \(\phi=1-\frac{1}{2p^\eta}\), the mixture distribution is perfectly debiased, and we have
\begin{equation}
\label{3.2 eq: mixture debiased training bound}
\ell^\eta(\pibold) \leq 2\exp\Big(-\frac{2\pi_{inv}^2}{4 \sum_{i=1}^D \alpha_i(t)^2+1}\Big),
\end{equation}
which is equivalent to the test bound in (\ref{supple eq: test upper bound}).
\end{customthm}

\begin{proof}


{Recall that $Z_{sp, i}^\eta$ follows the mixture distribution $P^\eta_{mix}$:
\begin{equation}
\label{supple eq: mixture distribution}
P_{mix}^{\eta}({Z}_{sp, i}^{\eta} \mid Y^\eta=y) = \phi P_{debias}^{\eta}({Z}_{sp, i}^{\eta} \mid Y^{\eta}=y) + (1-\phi) P_{bias}^{\eta}({Z}_{sp, i}^{\eta} \mid Y^{\eta}=y),
\end{equation}}

where 
\begin{equation}
\label{supple eq: mixture_debias}
P_{debias}^{\eta}({Z}_{sp, i}^{\eta}\mid Y^{\eta}=y)= 
\begin{cases}
    1, & \text{if } {Z}_{sp, i}^{\eta} = -y\\
    0, & \text{if } {Z}_{sp, i}^\eta = y 
\end{cases}
\end{equation} 
{is a debiasing distribution to weaken the correlation between \(Y^\eta\) and \({Z}_{sp, i}^\eta\) by setting the value of \(Z_{sp,i}^\eta\) as \(-Y^\eta\), and}
\begin{equation}
\label{supple eq: mixture_bias}
P_{bias}^\eta({Z}_{sp, i}^\eta\mid Y^\eta=y)= 
\begin{cases}
    p^\eta, & \text{if } {Z}_{sp, i}^\eta = y\\
    1-p^\eta, & \text{if } {Z}_{sp, i}^\eta = -y.
\end{cases}
\end{equation}

{Then, with definition in (\ref{supple eq: mixture_debias}) and (\ref{supple eq: mixture_bias}),
\begin{equation}
\label{supple eq: pmix}
\begin{split}
    P_{mix}({Z}_{sp, i}^\eta = -y | Y^\eta=y) &= \phi + (1-\phi)(1-p^\eta) \\
    P_{mix}({Z}_{sp, i}^\eta = y | Y^\eta=y) &= (1-\phi)p^\eta,
\end{split}
\end{equation}
for \(y \in \{-1, 1\}\). Then, based on the assumption, \(\phi+(1-\phi)(1-p^\eta) \leq (1-\phi)p^\eta\), which gives $\phi \leq 1-\frac{1}{2p^\eta}$. Specifically, if $\phi=1-\frac{1}{2p^\eta}$, it turns out that $P_{mix}({Z}_{sp, i}^\eta = -y | Y^\eta=y) = P_{mix}({Z}_{sp, i}^\eta = y | Y^\eta=y) = \frac{1}{2}$, which implies that spurious features turns out to be random and the mixture distribution becomes perfectly debiased. If $\phi = 0$, the mixture distribution boils down into a biased distribution as similarly defined in the environment $e \in \mathcal{E}_{train}.$}

The prediction from the classifier $\Oc^\eta$ is defined as similar to $\Oc^e$ in (\ref{supple eq: Oc}). Then in order to derive a concentration inequality of $\ell^\eta(\pibold)$, we derive a conditional expectation of \(\Oc^\eta\) as done in (\ref{supple eq: expectation of Z bar}):
\begin{equation}
\label{supple eq: mixture conditional expectation1}
\begin{split}
    \mathbb{E}_{\boldsymbol{X}^\eta, \boldsymbol{m}} \big[ \Oc^\eta \given[\big] Y^\eta = 1 \big] &= \mathbb{E}_{\boldsymbol{X}^\eta, \boldsymbol{m}} \Big[ \tilde{w}_{inv} m_{inv} Z_{inv}^\eta + \sum_{i=1}^D \tilde{w}_{sp, i} m_{sp, i} Z_{sp, i}^\eta  \given[\Big] Y^\eta = 1 \Big] \\
    &= \mathbb{E}_{\boldsymbol{X}^\eta, \boldsymbol{m}} \Big[ \tilde{w}_{inv} m_{inv} + \sum_{i=1}^D \tilde{w}_{sp, i} m_{sp, i} Z_{sp, i}^\eta  \given[\Big] Y^\eta=1 \Big].
\end{split}
\end{equation}
Then, with the definition in (\ref{supple eq: mixture distribution}), the second term in the above conditional expectation of (\ref{supple eq: mixture conditional expectation1}) is defined as follows: 
\begin{equation}
\label{supple eq: mixture conditional expectation2}
\begin{split}
    \mathbb{E}_{\boldsymbol{X}^\eta, \boldsymbol{m}} \Big[ \sum_{i=1}^D & \tilde{w}_{sp, i} m_{sp, i} Z_{sp, i}^\eta \given[] Y^\eta = 1 \Big] \\
     &= \sum_{i=1}^D \tilde{w}_{sp, i} \pi_{sp, i} \Big(\phi \mathbb{E}_{debias}[Z_{sp, i}^\eta \given[] Y^\eta = 1] + (1-\phi)\mathbb{E}_{bias}[Z_{sp, i}^\eta \given[] Y^\eta = 1] \Big) \\
     &= \sum_{i=1}^D \tilde{w}_{sp, i} \pi_{sp, i} \Big(  \phi \cdot (-1) + (1-\phi)(2p^\eta-1)  \Big) \\
     &= \sum_{i=1}^D \tilde{w}_{sp, i} \pi_{sp, i} \big(2p^\eta(1-\phi)-1\big),
\end{split}
\end{equation}
where \(\mathbb{E}_{debias}\) and \(\mathbb{E}_{bias}\) in the first equality denote the conditional expectation with respect to distribution \(P_{debias}^\eta\) and \(P_{bias}^\eta\) in (\ref{supple eq: mixture_debias}) and (\ref{supple eq: mixture_bias}), respectively. Plugging (\ref{supple eq: mixture conditional expectation2}) into (\ref{supple eq: mixture conditional expectation1}), we get
\begin{equation}
\label{supple eq: mixture conditional expectation3}
    \mathbb{E}_{\boldsymbol{X}^\eta, \textbf{m}} \big[ \Oc^\eta \given[\big] Y^\eta = 1 \big] = \tilde{w}_{inv} \pi_{inv} + \sum_{i=1}^D \big(2p^\eta(1-\phi)-1\big) \tilde{w}_{sp, i} \pi_{sp, i}.
\end{equation}
Then we can derive the upper bound of $\ell^\eta(\pibold) = P(\Oc^\eta < 0 \given Y^\eta = 1)$ similarly to (\ref{supple eq: naive hoeffding}):
\begin{equation}
\label{supple eq: mixture hoeffding}
\begin{split}
    P\big( \Oc^\eta < 0 \given[\big] Y^\eta=1 \big) &\leq P\Big( \given[\Big] \Oc^\eta - \mathbb{E}_{\boldsymbol{X}^\eta, \boldsymbol{m}} \big[ \Oc^\eta \big] \given[\Big] > \mathbb{E}_{\boldsymbol{X}^\eta, \boldsymbol{m}} \big[ \Oc^\eta \big]  \given[\big] Y^\eta=1 \Big) \\
    &\leq  2\exp\bigg(-\frac{2 \mathbb{E}_{\boldsymbol{X}^\eta, \boldsymbol{m}} \big[ \Oc^\eta \given[\big] Y^\eta = 1 \big]^2 }{\tilde{w}_{inv}^2 + 4 \sum_{i=1}^D \tilde{w}_{sp, i}^2}\bigg) \\
    &\leq 2\exp\bigg(-\frac{2 \big(\tilde{w}_{inv} \pi_{inv} + \sum_{i=1}^D \big(2p^\eta(1-\phi)-1\big) \tilde{w}_{sp, i} \pi_{sp, i} \big)^2 }{\tilde{w}_{inv}^2 + 4 \sum_{i=1}^D \tilde{w}_{sp, i}^2}\bigg) \\
    &\leq  2\exp\Big(-\frac{2 \big(\pi_{inv} + \sum_{i=1}^D   (2p^\eta(1-\phi)-1) \alpha_i \pi_{sp, i} \big)^2}{1 + \sum_{i=1}^D 4\alpha_{i}^2}\Big),
\end{split}
\end{equation}
where the first inequality is obtained by Hoeffding's inequality, and second inequality is from (\ref{supple eq: mixture conditional expectation3}). The denominator is obtained as same as in (\ref{supple eq: naive hoeffding}), since $\tilde{w}_{inv} m_{inv} Z^\eta_{inv} \in \{0, \tilde{w}_{inv}\}$ and $ \tilde{w}_{sp, i} m_{sp, i} Z^\eta_{sp, i} \in \{-\tilde{w}_{sp, i}, 0, \tilde{w}_{sp, i}\} ~\forall i$ as-is. If we plug-in the upper bound value of $\phi = 1 - \frac{1}{2p^\eta}$ obtained from (\ref{supple eq: pmix}) into (\ref{supple eq: mixture hoeffding}), it boils down into the test bound in (\ref{supple eq: test upper bound}).
\end{proof}

\subsection{Dynamics of the weight ratio}
\label{sec: weight ratio}

We omit an index of environment $e$ in the proposition below for notational simplicity.

\begin{proposition}
\label{supple prop: weight ratio} 
    Consider a binary classification problem of linear classifier $f_{\wb}$ under exponential loss. Let $(\Xb, Y) \sim P$, where each input random variable $\Xb$ and the corresponding label $Y$ is generated by 
    $$ \Xb = \begin{pmatrix} Z_{inv} \\ \Zb_{sp} \end{pmatrix}, 
            Y = Z_{inv}, $$ 
    where $\Zb_{sp} = (2\zb - 1)Z_{inv}$ for a random variable $\zb \in \{0, 1\}^D$ which is chosen from multivariate Bernoulli distribution ($z_i \sim Bern(p)$) with $p>\frac12$, i.e., $p$ denotes $p^e$ in the main paper.
    Let $\wb=\begin{pmatrix} w_{inv} \\ \wb_{sp} \end{pmatrix} \in \Rd^{D+1}$ be the weight of the linear classifier $f_{\wb}(\xb)=\wb^T\xb$. Assume that $0 < \winv(0)$, i.e., $\winv$ is initialized with a positive value, and $0 < w_{sp,i}(0) < \frac12 \log \frac{p}{1-p}$.   
    Then, after sufficient time of training, $\winv$ diverges to $+\infty$ and $w_{sp,i}$ converges to $\frac12 \log \frac{p}{1-p}$, which means
    $\alpha_i := \frac{w_{sp,i}}{\winv}$ converges to $0$ for all $i \in \{1,2,\cdots,D\}$. More precisely,
    $$ 
    \log \left( e^{\winv(0)} + [4p(1-p)]^{\frac{D}{2}} t \right) \le \winv(t) \le 
    \log \left( e^{w_{inv}(0)} + t \prod_{i=1}^D \left( pe^{-w_{sp,i}(0)}+ \sqrt{p(1-p)}\right) \right).
    $$
    However, for a fixed $t < T$, each $\alpha_i$ is positive and its lower bound converges to some positive value.
\end{proposition}

\begin{proof}
In this proof, $\winv(t)$ denotes the invariant weight at time $t$, while we often omit the time $t$ and interchangeably use $\winv$ for notational simplicity, and likewise for $w_{sp, i}(t)$.

Note that the network output is given by
    \begin{align*}
        f_{\wb}(\xb) &= \wb^T\xb \\
        &= \zinv w_{inv} + \zsp^T \wb_{sp} \\
        &= \zinv w_{inv} + \sum_{i=1}^D Z_{sp,i}w_{sp,i}.
    \end{align*}
    The exponential loss is defined by
    \begin{align*}
        L(\wb) &= \Ed_{(\Xb, Y)} [ e^{-f_{\wb}(\Xb)Y} ] \\
        &= \Ed_{\zb} \big[\exp\big({- (\zinv w_{inv} + \sum_{i=1}^D Z_{sp,i}w_{sp,i})\zinv} \big) \big]\\
        &= \Ed_{\zb} \big[\exp({-\winv - (2z_1-1)w_{sp,1} - \cdots - (2z_{D}-1)w_{sp,D} })\big] \\
        &= e^{-\winv} \prod_{i=1}^D \Ed_{\zb} [ e^{-(2z_{i}-1)w_{sp,i}} ] \\
        &= e^{-\winv} \prod_{i=1}^D ( pe^{-w_{sp,i}} + (1-p)e^{w_{sp,i}}).
    \end{align*}
    
    Then, thanks to symmetry of $\wb_{sp}$, it is enough to consider $\alpha := \frac{w_{sp,1}}{\winv}$. We first compute the gradient:
    \begin{align*}
        \frac{\partial L}{\partial \winv}
        &= -e^{-\winv} \prod_{i=1}^D (pe^{-w_{sp,i}}+(1-p)e^{w_{sp,i}}) \\
        \frac{\partial L}{\partial w_{sp,1}} 
        &= -e^{-\winv} (pe^{-w_{sp,1}}-(1-p)e^{w_{sp,1}}) 
        \prod_{i=2}^D (pe^{-w_{sp,i}}+(1-p)e^{w_{sp,i}}) .
    \end{align*}
    
    

    Since $\frac{d}{dt} w_{inv} = -\frac{\partial L}{\partial \winv}$, the dynamics is given by the following differnetial equations.
    \begin{align*}
        \frac{d}{dt} w_{inv} 
        &= e^{-\winv} \prod_{i=1}^D (pe^{-w_{sp,i}}+(1-p)e^{w_{sp,i}}) \\
        \frac{d}{dt} w_{sp,1}
        &= e^{-\winv} (pe^{-w_{sp,1}}-(1-p)e^{w_{sp,1}}) 
        \prod_{i=2}^D (pe^{-w_{sp,i}}+(1-p)e^{w_{sp,i}}) .
    \end{align*}
    
    First we show that $\winv(t)$ diverges to $+\infty$ as $t$ goes $\infty$. We show this by computing its lower bound.
    \begin{align*}
        \frac{d}{dt} \winv 
        &= e^{-\winv} \prod_{i=1}^D (pe^{-w_{sp,i}}+(1-p)e^{w_{sp,i}}) \\ 
        &\ge e^{-\winv} \prod_{i=1}^D (2\sqrt{p(1-p)}) \\
        &= e^{-\winv} [4p(1-p)]^{\frac{D}{2}},
    \end{align*}
    where the inequality is obtained by AM-GM inequality. This implies $e^{\winv} d\winv \ge [4p(1-p)]^{\frac{D}{2}} dt$. Integrating both sides from $0$ to $t$, we get
    \begin{align*}
        e^{\winv(t)} - e^{\winv(0)} \ge [4p(1-p)]^{\frac{D}{2}} t
    \end{align*}
    or
    \begin{align} \label{eq: winv lower bound}
        \winv(t) \ge \log \left( e^{\winv(0)} + [4p(1-p)]^{\frac{D}{2}} t \right),
    \end{align}
    which shows that $\winv(t)$ diverges to $+\infty$ as $t \rightarrow \infty$. Note also that $w_{inv}$ strictly increases since $\frac{d}{dt} w_{inv}>0$.
    
    For $w_{sp,i}$, $\frac{d}{dt}w_{sp,i} =0$ implies $w_{sp,i}$ converges to $w_{sp,i}^*$ such that
    \begin{align*}
        pe^{-w_{sp,i}^*}-(1-p)e^{w_{sp,i}^*} = 0,
    \end{align*}
    namely, $w_{sp,i}^* = \frac12 \log \frac{p}{1-p}$.
    
    As similar to $\winv$, $w_{sp,1}$ strictly increases if and only if $w_{sp,1}<\frac12 \log \frac{p}{1-p}$. Based on the assumptions that $0 < w_{sp,i}(0) < \frac12 \log \frac{p}{1-p}$, we conclude that $w_{sp,1}$ monotonically converges to $\frac12 \log \frac{p}{1-p}$. As $p$ goes to $1$, $\frac12 \log \frac{p}{1-p}$ is sufficiently large and we can assume $w_{sp,i}(0) < \frac12 \log \frac{p}{1-p}$. 
    
    Now, we fix $0<t<T$ for given $T$ and compute an upper bound of $w_{inv}$. Using $w_{sp,i}(t)<\frac12 \log \frac{p}{1-p}$, we get
    \begin{align*}
        \frac{d}{dt} w_{inv} 
        &= e^{-\winv} \prod_{i=1}^D (pe^{-w_{sp,i}}+(1-p)e^{w_{sp,i}}) \\
        &< e^{-\winv} \prod_{i=1}^D \left(pe^{-w_{sp,i}(0)} + (1-p)\sqrt{\frac{p}{1-p}} \right) \\
        &= e^{-\winv} \prod_{i=1}^D \left( pe^{-w_{sp,i}(0)} + \sqrt{p(1-p)} \right) 
    \end{align*}
    which implies
    \begin{align*}
        e^{\winv} d w_{inv} < \prod_{i=1}^D \left( pe^{-w_{sp,i}(0)} + \sqrt{p(1-p)} \right)  dt.
    \end{align*}
    Integrating both sides from $0$ to $t$, we get
    \begin{align} \label{eq: winv bound}
        w_{inv}(t) < \log \left( e^{w_{inv}(0)} + \prod_{i=1}^D \left( pe^{-w_{sp,i}(0)}+ \sqrt{p(1-p)}\right) t \right).
    \end{align}
    
    Similarly, we compute a lower bound of $w_{sp,1}$ on $0<t<T$. Before we start, note that $\winv(t) < \winv(T)=:M$ from monotonicity. 
    \begin{align*}
        \frac{d}{dt}w_{sp,1} &= e^{-\winv} (pe^{-w_{sp,1}}-(1-p)e^{w_{sp,1}}) 
        \prod_{i=2}^D (pe^{-w_{sp,i}}+(1-p)e^{w_{sp,i}}) \\
        &> e^{-M} (pe^{-w_{sp,1}}-(1-p)e^{w_{sp,1}}) \prod_{i=2}^D (2\sqrt{p(1-p)}) \\
        &= e^{-M} [4p(1-p)]^{\frac{D-1}{2}} (pe^{-w_{sp,1}}-(1-p)e^{w_{sp,1}})
    \end{align*}
    induces
    \begin{align*}
        \frac{1}{pe^{-w_{sp,1}}-(1-p)e^{w_{sp,1}}} d w_{sp,1}
        > e^{-M} [4p(1-p)]^{\frac{D-1}{2}} dt .
    \end{align*}
    Integrating both sides from $0$ to $t<T$, we get
    \begin{align*}
        \left[\frac{1}{\sqrt{p(1-p)}}\tanh^{-1} \left( \sqrt{\frac{1-p}{p}} e^{w_{sp, 1}}\right) \right]_0 ^t
        >  e^{-M} [4p(1-p)]^{\frac{D-1}{2}} t
    \end{align*}
    or
    \begin{align} \label{eq: wsp bound}
        w_{sp,1}(t) > \frac12 \log \frac{p}{1-p} + \log \; \tanh \left( 
        \tanh^{-1}(\sqrt{\frac{1-p}{p}} e^{w_{sp,1}(0)}) + e^{-M} 2^{D-1} [p(1-p)]^{\frac D2} t \right). 
    \end{align}
    Combining \eqref{eq: winv bound} and \eqref{eq: wsp bound}, we conclude that 
    \begin{align} \label{eq: alpha bound}
        \alpha_p (t) &= \frac{w_{sp,1}(t)}{w_{inv}(t)} \\
        &> \frac{\frac12 \log \frac{p}{1-p} + \log \; \tanh \left( 
        \tanh^{-1}(\sqrt{\frac{1-p}{p}} e^{w_{sp,1}(0)}) + e^{-M} 2^{D-1} [p(1-p)]^{\frac D2} t \right) }{\log \left( e^{w_{inv}(0)} + t \prod_{i=1}^D \left( pe^{-w_{sp,i}(0)}+ \sqrt{p(1-p)}\right) \right)}
    \end{align}
    for $0<t<T$. Note that $\alpha_p(t)$ is positive in $0<t<T$, since both $w_{sp, 1}(t)$ and $\winv(t)$ is monotonically increasing in $0<t<T$, and $0 < w_{sp, 1}(0), \winv(0)$ by assumptions.
    
    The numerator becomes
    \begin{align*}
        \frac12 &\log \frac{p}{1-p} + \log \; \tanh \left( 
        \tanh^{-1}(\sqrt{\frac{1-p}{p}} e^{w_{sp,1}(0)}) + e^{-M} 2^{D-1} [p(1-p)]^{\frac D2} t \right) \\
        &= \log \left[ \sqrt{\frac{p}{1-p}} \tanh \left( 
        \tanh^{-1}(\sqrt{\frac{1-p}{p}} e^{w_{sp,1}(0)}) + e^{-M} 2^{D-1} [p(1-p)]^{\frac D2} t \right) \right] \\
        &= \log \left[ \sqrt{\frac{p}{1-p}} \left( 
        \sqrt{\frac{1-p}{p}} e^{w_{sp,1}(0)} + e^{-M} 2^{D-1} [p(1-p)]^{\frac D2} t \;\sech^2 c \right) \right]
    \end{align*}
    for some $c$ such that 
    \begin{align*}
        \tanh^{-1}(\sqrt{\frac{1-p}{p}} e^{w_{sp,1}(0)}) < c < \tanh^{-1}(\sqrt{\frac{1-p}{p}} e^{w_{sp,1}(0)}) + e^{-M} 2^{D-1} [p(1-p)]^{\frac D2} t.
    \end{align*}
    We use $f(x+y) = f(x) + y f'(c)$ by the Mean Value Theorem (MVT) at the last line. 
    
    Notably, if we take a limit $p \rightarrow 1$, the numerator becomes
    \begin{align*}
        \lim_{p\rightarrow 1} \log \left[ e^{w_{sp,1}(0)} + e^{-M} 2^{D-1} p^{\frac{D+1}{2}} (1-p)^{\frac{D-1}{2}}  t\;\sech^2 c \right]
        = w_{sp,1}(0).
    \end{align*}
    Similarly, the denominator becomes
    \begin{align*}
        \lim_{p\rightarrow1} \log &
        \left( e^{w_{inv}(0)} + t\prod_{i=1}^D \left( pe^{-w_{sp,i}(0)}+ \sqrt{p(1-p)}\right)  \right) \\
        &= \log \left( e^{w_{inv}(0)} + t\prod_{i=1}^D e^{-w_{sp,i}(0)} \right) \\
        &= \log \left( e^{w_{inv}(0)} + t \exp\left(-\sum_{i=1}^D w_{sp,i}(0)\right) \right)
    \end{align*}
    
    Therefore, for a fixed $0<t<T$, we conclude that 
    \begin{equation}
    \label{eq: alphap(t)}
    \begin{split}
        \lim_{p\rightarrow1} \alpha_p(t) &= \lim_{p\rightarrow1} \frac{w_{sp,1}(t)}{\winv(t)} \\ 
        &\ge \frac{w_{sp,1}(0)}{\log \left( e^{w_{inv}(0)} +  t \exp\left(-\sum_{i=1}^D w_{sp,i}(0)\right) \right)} \\
        &> \frac{w_{sp,1}(0)}{\log(e^{\winv(0)} +  T \exp\left(-\sum_{i=1}^D w_{sp,i}(0)\right))} \\
        & \ge \frac{w_{sp,1}(0)}{\log T + \frac1T \exp\left(\winv(0)+\sum_{i=1}^D w_{sp,i}(0)\right) - \sum_{i=1}^D w_{sp,i}(0)} 
    \end{split}
    \end{equation}
    where we use the inequality $\log (x+y) \le \log x + \frac yx$ in the last line.
\end{proof}

{The key insights from Proposition \ref{supple prop: weight ratio} can be summarized as follows:}

(1) Weight ratio $\alpha_i(t)$ converges to $0$ as $t \rightarrow \infty$.

(2) However, for a fixed $t < T$, $\alpha_i(t) > 0$.

(3) When $t < T$ and $p \rightarrow 1$, i.e., the environment is almost perfectly biased, the convergence rate of (1) is remarkably slow as in (\ref{eq: alphap(t)}). In other words, there exists $c>0$ such that $\frac{c}{\log t} < \alpha_p(t)$ over $0<t<T$ if $p$ is sufficiently close to $1$.

{This results afford us intriguing perspective on the fundamental factors behind the biased classifiers. If we situate the presented theoretical example in an ideal scenario in which infinitely many data and sufficient training time is provided, our result (1) shows that the pretrained classifier becomes fully invariant to the spurious correlations. However, in practical setting with finite training time and number of samples, our result (2) shows that the pretrained model inevitably rely on the spuriously correlated features.}

\begin{wrapfigure}[14]{r}{0.4\textwidth}
\includegraphics[width=0.4\textwidth]{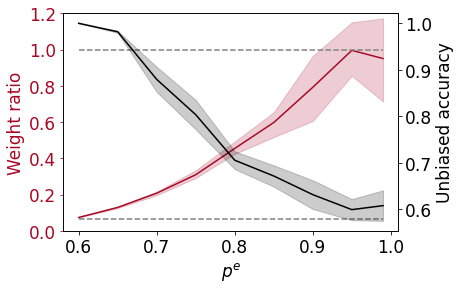}
\caption{Implemented results of presented example.}
\label{supple fig: alpha}
\end{wrapfigure}

~{Beyond theoretical results, we empirically observe that the weight ratio $\alpha_i$ of pretrained classifiers indeed increases as $p^e \rightarrow 1$. We simulate the example presented in section 3.2 of the main paper, where the dimensionality $D$ is set to 15, and probability $p^e$ varies from 0.6 (weakly biased) to 0.99 (severely biased). We train a linear classifier for 500 epochs with batch size of 1024, and measure the unbiased accuracy on test samples generated from environment $e \in \mathcal{E}_{test}$. We also measure weight ratio $ \mean(\tilde{w}_{sp}) / \tilde{w}_{inv}$, where $\mean(\bar{w}_{sp})$ denotes the average of pretrained spurious weights $\{w_{sp, i}\}_{i=1}^{D}$. 
To enable the end-to-end training, we use binary cross entropy loss instead of exponential loss, with setting $\mathcal{Y} = \{0, 1\}$ instead of $\mathcal{Y} = \{-1, 1\}$. We do not consider pruning process in this implementation. Figure \ref{supple fig: alpha} shows that the weight ratio increases to 1 in average as $p^e \rightarrow 1$. It implies that the spurious features $\Zb_{sp}^e$ participate almost equally to the invariant feature $\zinv^e$ in the presence of strong spurious correlations. In this worst case, it is frustratingly difficult to discriminate weights necessary for OOD generalization in biased environment, resulting in the failure of learning optimal pruning parameters. Simulation results are averaged on 15 different random seeds.}

\section{Example of geometrical misalignment}
\label{sec: misalignment}

In this section, we present a simple example illustrating the potential adverse effect of spurious correlations on latent representations. Consider independent arbitrary samples within the same class \(\boldsymbol{X}_i^{b}, \boldsymbol{X}_j^{b} \sim P^{b}_{\boldsymbol{X}^b \mid Y^b=y}\) and \(\boldsymbol{X}^{d} \sim P^{d}_{\boldsymbol{X}^d \mid Y^d=y} \) for a common \(y \in \{-1, 1\}\) and environments \(b, d\) where \(b \in \mathcal{E}_{train}\) and \(d \in \mathcal{E}_{test}\). Let \(\boldsymbol{W} \in \mathbb{R}^{Q \times (D+1)}\) be a weight matrix representation of a linear mapping \(T: \{-1, 1\}^{D+1} \rightarrow \mathbb{R}^Q\) which encodes the embedding vector of a given sample. We denote such embedding as \(\boldsymbol{h}^e = \boldsymbol{W}\boldsymbol{X}^e\) for some \(e \in \mathcal{E}\). We assume that \(\boldsymbol{W}\) is initialized as to be semi-orthogonal \cite{saxe2013exact, hu2020provable} for simplicity. Then the following lemma reveals the geometrical misalignment of embeddings in the presence of strong spurious correlations:

\begin{lemma}
\label{supple lemma: misalignment}
Given \(y \in \{-1, 1\}\), let \(\boldsymbol{h}^b_i, \boldsymbol{h}^b_j, \boldsymbol{h}^d\) be embeddings of \(\boldsymbol{X}_i^{b}, \boldsymbol{X}_j^{b}, \boldsymbol{X}^{d}\) respectively. Then, the expected cosine similarity between \(\boldsymbol{h}^b_i\) and \(\boldsymbol{h}^d\) is derived as:
\begin{equation}
\label{supple eq: cos(b,d)}
    \mathbb{E}\bigg[ \frac{\langle \boldsymbol{h}^{b}_i, \boldsymbol{h}^{d} \rangle}{\|\boldsymbol{h}^{b}_i\| \cdot \|\boldsymbol{h}^{d}\| } \given[\bigg] Y^b=y, Y^d=y \bigg] = \frac{1}{D+1},
\end{equation}
while the expected cosine similarity between \(\boldsymbol{h}^b_i\) and \(\boldsymbol{h}^b_j\) is derived as:
\begin{equation}
\label{supple eq: cos(b,b)}
    \mathbb{E}\bigg[ \frac{\langle \boldsymbol{h}^{b}_i, \boldsymbol{h}^{b}_j \rangle}{\|\boldsymbol{h}^{b}_i\| \cdot \|\boldsymbol{h}^{b}_j\| } \given[\bigg] Y^b=y \bigg] = \frac{1+D(2p^b-1)^2}{D+1},
\end{equation}
where \(p^b\) is a probability parameter of Bernoulli distribution of i.i.d variable \({Z}_{sp,i}^b\), similar to \(p^e\) in the main paper.
\end{lemma}

\begin{proof}
Let \(\boldsymbol{X}^e=\boldsymbol{V}_{inv}^e + \boldsymbol{V}_{sp}^e\) for the sample from an arbitrary environment \(e\) in general, where \(\boldsymbol{v}_{inv}^e, \boldsymbol{v}_{sp}^e \in \{-1,1\}^{D+1}\) are invariant and spurious component vector, respectively:

\begin{equation}
\label{supple eq: v_inv}
{V}_{inv, j}^e= 
\begin{cases}
    Z^e_{inv}, & \text{if } j=1 \\
    0, & \text{otherwise } ,
\end{cases}
\end{equation}
\begin{equation}
\label{supple eq: v_sp}
{V}_{sp, j}^e=
\begin{cases}
    {Z}^e_{sp, j}, & \text{if } j=2, \dots, D+1 \\
    0, & \text{otherwise }.
\end{cases}
\end{equation} 
Thus, \(\boldsymbol{V}_{inv}^e\) and \(\boldsymbol{V}_{sp}^e\) are orthogonal. Given \(Y^b=y\) and \(Y^d=y\) for some \(y \in \{-1, 1\}\), the cosine similarity between \(\boldsymbol{h}_i^{b}\) and \(\boldsymbol{h}^{d}\) is expressed as follows:

\begin{equation}
\label{supple eq: cos_sim_bd}
\begin{split}
    \mathbb{E}\bigg[ \frac{\langle \boldsymbol{h}_i^{b}, \boldsymbol{h}^{d} \rangle}{\| \boldsymbol{h}_i^{b} \| \| \boldsymbol{h}^{d} \| } \given[\bigg] Y^b=y, Y^d=y \bigg] &= \mathbb{E}\bigg[ \frac{\langle \boldsymbol{X}_i^{b}, \boldsymbol{W}^T \boldsymbol{W} \boldsymbol{X}^{d} \rangle}{\| \boldsymbol{h}_i^{b} \| \| \boldsymbol{h}^{d} \| } \given[\bigg] Y^b=y, Y^d=y \bigg] \\
    &= \mathbb{E}\bigg[ \frac{\langle \boldsymbol{X}_i^{b}, \boldsymbol{X}^{d} \rangle}{D+1} \given[\bigg]  Y^b=y, Y^d=y \bigg] \\
    &= \mathbb{E} \bigg[ \frac{\langle \boldsymbol{V}^{b}_{i, inv}+\boldsymbol{V}^{b}_{i, sp}, \boldsymbol{V}^{d}_{inv}+\boldsymbol{V}^{d}_{sp} \rangle}{D+1}  \given[\bigg] Y^b=y, Y^d=y \bigg] \\
    &= \frac{1}{D+1},
\end{split}
\end{equation}
where \(\boldsymbol{V}_{i, inv}^b\) and \(\boldsymbol{V}_{i, sp}^b\) represent the invariant and spurious component vector of \(\boldsymbol{X}_i^b\), respectively, and the second equality comes from the semi-orthogonality of \(\boldsymbol{W}\). The last equality comes from the orthogonality of spurious component vector from different environment \(b \in \mathcal{E}_{train}\) and \(d \in \mathcal{E}_{test}\).

On the other hand, the expected cosine similarity between two arbitrary embeddings \(\boldsymbol{h}_i^{b}\) and \(\boldsymbol{h}_j^{b}\) from the biased environment \(b\) is expressed as follows:
\begin{equation}
\label{supple eq: cos_sim_bb}
\begin{split}
    \mathbb{E}\bigg[ \frac{\langle \boldsymbol{h}_i^{b}, \boldsymbol{h}_j^{b} \rangle}{\| \boldsymbol{h}_i^{b} \| \| \boldsymbol{h}_j^{b} \| } \given[\bigg] Y^b=y \bigg] 
    &= \mathbb{E} \bigg[ \frac{\langle \boldsymbol{V}^{b}_{i, inv} + \boldsymbol{V}^{b}_{i, sp}, \boldsymbol{V}^{b}_{j, inv} + \boldsymbol{V}^{b}_{j, sp} \rangle}{D+1}  \given[\bigg] Y^e=y\bigg] \\
    &= \frac{1+D(2p^b-1)^2}{D+1},
\end{split}
\end{equation}
where the last equality comes from the expectation of product of independent 
Bernoulli variables.
\end{proof}

The gap between (\ref{supple eq: cos(b,d)}) and (\ref{supple eq: cos(b,b)}) unveils the imbalance of distance between same-class embeddings from different environments on the unit hypersphere; embeddings from the training environment are more closely aligned to other embeddings from the same environment than embeddings from test environment at initial even when all samples are generated within the same class. While the Lemma \ref{supple lemma: misalignment} is only applicable to the initialized \(\boldsymbol{W}\) before training, such imbalance may be worsened if \(\boldsymbol{W}\) learns to project the samples on the high-dimensional subspace where most of its basis are independent to the invariant features. This sparks interests in designing weight pruning masks to aggregate the representations from same-class samples all together. Indeed, in this simple example, we can address this misalignment by masking out every weight in \(\boldsymbol{W}\) except the first column, which is associated with the invariant feature.

{From this point of view, we revisit the proposed alignment loss in main paper:
\begin{equation}
\label{supple eq: debiased alignment loss}
\ell_{align}\Big( \{x_i, y_i\}_{i=1}^{|S|}; \boldsymbol{\tilde{W}}, \Theta \Big) = \mathbb{E}_{\boldsymbol{m} \sim G(\Theta)}\Big[ \ell_{con}(S_{bc}, S; \boldsymbol{m} \odot \boldsymbol{\tilde{W}}) + \ell_{con}(S_{ba}, S_{bc}; \boldsymbol{m} \odot \boldsymbol{\tilde{W}}) \Big],
\end{equation}
where the first term reduces the gap between bias-conflicting samples and others, while the second term prevents bias-aligned samples from being aligned too close each other. In other words, the first term is aimed at increasing the cosine similarity between representations of same-class samples with different spurious attributes, as $\boldsymbol{h}_{i}^b$ and $\boldsymbol{h}^d$ in this example. The second term serves as a regularizer that pulls apart same-class bias-aligned representations, as $\boldsymbol{h}_{i}^b$ and $\boldsymbol{h}_{j}^b$ in this example. Thus we can leverage abundant bias-aligned samples as negatives regardless of their class in second term, while \cite{zhang2022correct} limits the negatives to samples with different target label but same bias label, which are often highly scarce in a biased dataset.}

\section{Additional results}
\label{sec: additional results}
\textbf{Comparisons to the pruning baselines.} Pruning (debiasing) appears to suffer from the generalization-efficiency tradeoff; improving computational efficiency (OOD generalization) does not always guarantee improvement in OOD generalization (efficiency). Unlike this, our framework reliably improves both generalization and efficiency as shown in Table \ref{supple table: pruning}. Note that the standard pruning algorithms \cite{wang2020picking} fail to improve the unbiased accuracy in CIFAR10-C.

\begin{table}[htbp]
\caption{Test (unbiased) accuracy ($\%$) on standard and corrupted CIFAR10 (Bias ratio=5\(\%\)). Pruning ratio=90.0$\%$ for GraSP and $(92.4\%, 90.4\%)$ for DCWP on (CIFAR10, CIFAR10-C).} 
\label{supple table: pruning}
\centering
\resizebox{0.4\textwidth}{!}{
\begin{tabular}{c c c c}
\toprule
{Dataset} & {Full-size} & {GraSP \cite{wang2020picking}} & {DCWP} \\
\midrule
{CIFAR10} & {86.76} & \cellcolor{OliveGreen!15} 85.64 & \cellcolor{OliveGreen!15} 86.32 \\
\midrule
{CIFAR10-C} & {45.66} & \cellcolor{BrickRed!15} 44.21 & \cellcolor{OliveGreen!15} 60.24 \\
\bottomrule
\end{tabular}
}
\end{table}

\textbf{Analysis of sparsity level.} One may concern that a trade-off between performance and sparsity may exist. For example, networks with mild sparsity may still be over-parameterized and thus not fully debiased, whereas networks with high sparsity do not have enough capacity to preserve the averaged accuracy. In order to investigate the trade-off, we measure the unbiased accuracy by explicitly controlling the pruning ratio with varying $\lambda_{\ell_1}$. Figure \ref{supple fig: tradeoff} shows that (1) the trade-off between performance and sparsity does exist, while (2) the proposed framework is reasonably tolerant to high sparsity in terms of generalization. We conjecture that such tolerance is owing to the \textit{prioritized} elimination of spurious weights; the networks can be compressed to a significant extent without hurting the generalization after pruning out the spurious weights. 

\begin{figure}[t]
\centering
\includegraphics[width=0.8\textwidth]{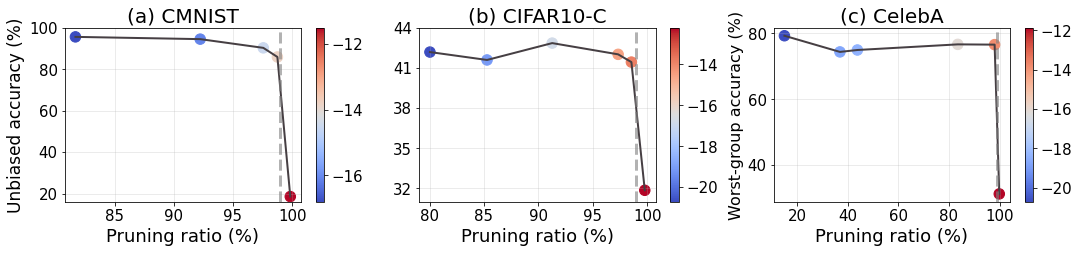} 
\caption{Analysis on the sparsity level. Bias ratio=1$\%$ for (a), (b). Color bar: log-scaled $\lambda_{\ell_1}$. Dotted line: $\text{ratio}=99\%$.}
\label{supple fig: tradeoff}
\end{figure}

\section{{Experimental setup}}
\label{sec: experimental setup}

\begin{figure}[htbp]
\centering
\includegraphics[width=0.9\textwidth]{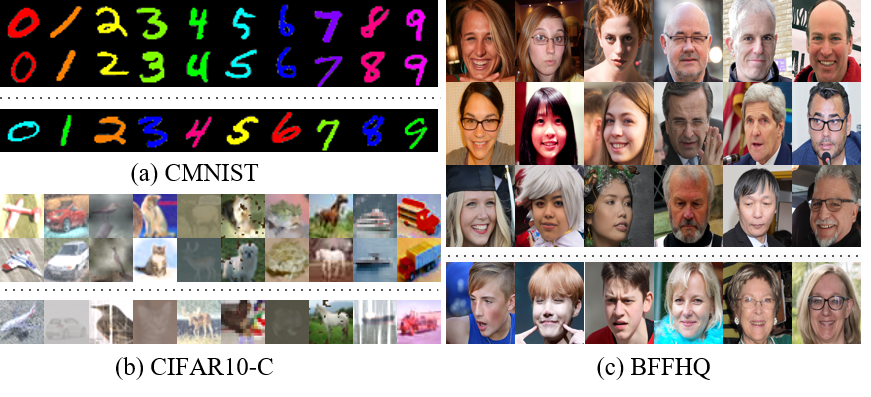} 
\caption{Example images of datasets. The images above the dotted line denote the bias-aligned samples, while the ones below the dotted line are the bias-conflicting samples. For CMNIST and CIFAR10-C, each column indicates each class. For BFFHQ, the group of three columns indicates each class.} \label{fig: examples}
\end{figure}

\subsection{Datasets} We mainly follow \cite{nam2020learning, lee2021learning} to evaluate our framework on Color-MNIST (CMNIST), Corrupted CIFAR-10 (CIFAR10-C) and Biased FFHQ (BFFHQ) as presented in Figure \ref{fig: examples}.  

\textbf{CMNIST.} We first consider the prediction task of digit class which is spuriously correlated to the pre-assigned color, following the existing works \cite{bahng2020learning, nam2020learning, lee2021learning, tartaglione2021end}. Each digit is colored with certain type of color, following \cite{nam2020learning, lee2021learning}. The ratio of bias-conflicting samples, i.e., bias ratio, is varied in range of $\{0.5\%, 1.0\%, 2.0\%, 5.0\%\}$, where the exact number of (bias-aligned, bias-conflicting) samples is set to:  (54,751, 249)-0.5$\%$, (54,509, 491)-1$\%$, (54,014, 986)-2$\%$, and (52,551, 2,449)-5$\%$.

\textbf{CIFAR10-C.} Each sample in this dataset is generated by corrupting original samples in CIFAR-10 with certain types of corruption. Among 15 different corruptions introduced in the original paper \cite{hendrycks2019benchmarking}, we select 10 types which are \texttt{Brightness, Contrast, Gaussian Noise, Frost, Elastic Transform, Gaussian Blur, Defocus Blur, Impulse Noise, Saturate}, and \texttt{Pixelate}, following \cite{lee2021learning}. Each of these corruption is spuriously correlated to the object classes of CIFAR-10, which are \texttt{Plane, Car, Bird, Cat, Deer, Dog, Frog,
Horse, Ship}, and \texttt{Truck}. 
We use the samples corrupted in most severe level among five different severity, following \cite{lee2021learning}. The exact number of (bias-aligned, bias-conflicting) samples is set to:  (44,832, 228)-0.5$\%$, (44,527, 442)-1$\%$,
(44,145, 887)-2$\%$, and (42,820, 2,242)-5$\%$.

\textbf{BFFHQ.} Each sample in this biased dataset are selected from Flickr-Faces-HQ (FFHQ) Dataset \cite{karras2019style}, where we conduct binary classifications with considering (\texttt{Age}, \texttt{Gender}) as target and spuriously correlated attribute pair following \cite{kim2021biaswap, lee2021learning}. Specifically, majority of training images correspond to either young women (i.e., aged
10-29) or old men (i.e., aged 40-59). This dataset consists of 19,104 number of such bias-aligned samples and 96 number of bias-conflicting samples, i.e., old women and young men.

\textbf{CelebA.} For CelebA, we consider (\texttt{Blonde Hair}, \texttt{Male}) as (target, spurious) attribute pair, following \cite{nam2020learning, hong2021unbiased, sagawa2019distributionally}. Pixel resolutions and batch size are $256 \times 256$ and 128, respectively. The exact number of samples for the prediction task follows that from \cite{hong2021unbiased}.

\subsection{Simulation settings}

\textbf{Architecture details.} We use a simple convolutional network with three convolution layers for CMNIST, with feature map dimensions of 64, 128 and 256, each followed by a ReLU activation and a batch normalization layer following \cite{zhang2021can}. For CIFAR10-C and BFFHQ, we use ResNet-18 with pretrained weights provided in PyTorch \texttt{torchvision} implementations. Each convolutional network and ResNet-18 includes $1.3 \times 10^6$ and $2.2 \times 10^7$ number of parameters, respectively. 
We assign a pruning parameter for each weight parameter except bias in deep networks. Each of pruning parameter is initialized with value 1.5 so that the initial probability of preserving the corresponding weight is set to $\sigmoid(1.5) \approx 0.8$ in default.

\textbf{Training details.} We first train bias-capturing networks using GCE loss (q=0.7) for (CMNIST, BFFHQ, CelebA), with (2000, 10000, 10000) iterations, respectively. For CIFAR10-C, we use epoch-ensemble-based mining algorithms presented in \cite{zhao2021learning}, which select samples cooperated with an ensemble of predictions at each epoch to prevent overfitting. We use b-c score threshold $\tau=0.8$ and the confidence threshold $\eta=0.05$ as suggested in the original paper. 

Then, main networks are pretrained for 10000 iterations using an Adam optimizer with learning rates $0.01$, $0.001$, $0.001$ and $0.0001$ for CMNIST, CIFAR10-C, BFFHQ, and CelebA, respectively. 

We train pruning parameters for 2000 iterations using a learning rate $0.01$, upweighting hyperparameter $\lambda_{up}=80$ and a balancing hyperparameter $\lambda_{align}=0.05$ for each dataset. We use a Lagrangian multiplier $\lambda_{\ell_1}=10^{-8}$ for CMNIST, and $\lambda_{\ell_1}=10^{-9}$ for CIFAR10-C, BFFHQ and CelebA. Specifically, we set $\lambda_{\ell_1}$ by considering the size of deep networks, where we found that the value within range $\Oc(0.1*n^{-1})$ serves as a good starting point where $n$ is the number of parameters. 

After pruning, we finetune the networks with decaying learning rate to $0.001$ for CMNIST and $0.0005$ for others. We use $\lambda_{align}=0.05$ consistently. Then, we use $\lambda_{up}=80$ for BFFHQ, $\lambda_{up}=20$ for CelebA, and $\lambda_{up}=\{10, 30, 50, 80\}$ for CMNIST and CIFAR10-C with $\{0.5\%, 1.0\%, 2.0\%, 5.0\%\}$ of bias ratio, respectively. 

Considering the pruning as a strong regularization, we did not use additional capacity control techniques such as early stopping or strong \(\ell_2\) regularization presented in \cite{sagawa2020investigation, liu2021just}. 

\textbf{Data augmentations.} We did not use any kinds of data augmentations which may implicitly enforce networks to encode invariances. For the BFFHQ and CelebA dataset, we only apply random horizontal flip. For the CIFAR10-C dataset, we take $32\times32$ random crops from image padded by 4 pixels followed by random horizontal flip, following \cite{nam2020learning}. We do not use any kinds of augmentations in CMNIST.

\textbf{Baselines.} We use the official implementations of Rebias, LfF, DisEnt and JTT released by authors, and reproduce EnD and MRM by ourselves. For DisEnt, we use the official hyperparameter configurations provided in the original paper. We use $q=0.7$ for LfF as suggested by authors on every experiment. For Rebias, we use the official hyperparameter configurations for CMNIST, and train for 200 epochs using Adam optimizer with learning rate 0.001 and RBF kernel radius of 1 for other datasets. For MRM, we use $\lambda_{\ell_1}$ of $10^{-8}$ for CMNIST following the original paper, and $10^{-9}$ for the others. For EnD, we set the multipliers $\alpha$ for disentangling and $\beta$ for entangling to 1.

{\small
\bibliographystyle{ieee_fullname}
\bibliography{egbib}
}

\end{document}

%% file: macros.tex
\newcommand{\Zb}{{\blmath Z}}

\newcommand{\wb}{{\blmath w}}
\newcommand{\xb}{{\blmath x}}

\newcommand{\zb}{{\blmath z}}

\newcommand{\Rd}{{\mathbb R}}

\newcommand{\Ed}{{{\mathbb E}}}

\newcommand{\beq}{\begin{equation}}
\newcommand{\eeq}{\end{equation}}
\newcommand{\beqa}{\begin{eqnarray}}
\newcommand{\eeqa}{\end{eqnarray}}
